\documentclass{article}

\usepackage{adjustbox}
\usepackage{mathrsfs,amsthm,amsmath,amssymb,bbm,prettyref,hyperref,graphicx}
\usepackage{epstopdf}
\usepackage{array}

\usepackage{fancyhdr,setspace}

\usepackage[marginratio=1:1,height=600pt,width=460pt,tmargin=100pt]{geometry}
\usepackage[compress]{cite}

\usepackage{authblk}

\usepackage{algorithm2e,tikz}
\usepackage{algpseudocode}

\usepackage{hyperref}
\hypersetup{
  colorlinks = false,
}

\usepackage{stackengine}
\usepackage{multirow}

\usepackage[font=small,labelfont=bf]{caption}

\usepackage{float}
\usepackage{subcaption}

\newcommand{\tabincell}[2]{\begin{tabular}{@{}#1@{}}#2\end{tabular}}

\newcolumntype{C}[1]{>{\centering\arraybackslash$}p{#1}<{$}}

\newcommand{\R}{\mathbb{R}}

\newtheorem{theorem}{Theorem}

\newtheorem{lemma}{Lemma}

\newtheorem{proposition}{Proposition}

\usepackage{authblk}

\title{Graph Convolution with Low-rank Learnable Local Filters}

\author[1]{Xiuyuan Cheng}
\author[2]{Zichen Miao}
\author[2]{Qiang Qiu\thanks{Email: qqiu@purdue.edu}}
\affil[1]{Department of Mathematics, Duke University}
\affil[2]{School of Electrical and Computer Engineering, Purdue University}
\date{\vspace{-20pt}}

\begin{document}

\vspace{-20pt}
\maketitle

\begin{abstract}

Geometric variations like rotation, scaling, and viewpoint changes pose a significant challenge to visual understanding. One common solution is to directly model certain intrinsic structures, e.g., using landmarks. However, it then becomes non-trivial to build effective deep models, especially when the underlying non-Euclidean grid is irregular and coarse. Recent deep models using graph convolutions provide an appropriate framework to handle such non-Euclidean data, but many of them, particularly those based on global graph Laplacians, lack expressiveness to capture local features required for representation of signals lying on the non-Euclidean grid. The current paper introduces a new type of graph convolution with learnable low-rank local filters, which is provably more expressive than previous spectral graph convolution methods. The model also provides a unified framework for both spectral and spatial graph convolutions. To improve model robustness, regularization by local graph Laplacians is introduced. The representation stability against input graph data perturbation is theoretically proved, making use of the graph filter locality and the local graph regularization. Experiments on spherical mesh data, real-world facial expression recognition/skeleton-based action recognition data, and data with simulated graph noise show the empirical advantage of the proposed model.
\end{abstract}

\section{Introduction}

Deep methods have achieved great success in visual cognition, 
yet they still lack capability to tackle severe geometric transformations such as rotation, scaling and viewpoint changes.
This problem is often handled by conducting data augmentations with these geometric variations included, e.g. by randomly rotating images, so as to make the trained model robust to these variations. However, this would remarkably increase the cost of training time and model parameters. 
Another way is to make use of certain underlying structures of objects, e.g. facial landmarks  \cite{face_ldmk_blessing} and human skeleton landmarks  \cite{body_ldmk_liegroup},
c.f. Fig. \ref{fig:diag1} (right). 
Nevertheless, 
these methods then adopt hand-crafted features based on landmarks, 
which greatly constrains their ability to obtain rich features for downstream tasks.
One of the main obstacles for feature extraction
is the non-Euclidean property of underlying structures,
and particularly, 
it prohibits the direct usage of prevalent convolutional neural network (CNN) architectures \cite{resnet,densenet}. 
Whereas there are recent CNN models designed for non-Euclidean grids, 
e.g., for spherical mesh \cite{ugscnn, s2cnn,spherenet}
and manifold mesh in computer graphics \cite{ bronstein2017geometric, splinecnn}, 
they mainly rely on partial differential operators which only can be calculated precisely on fine and regular mesh, 
and may not be applicable to the landmarks which are irregular and course.
Recent works have also applied Graph Neural Network (GNN) approaches to coarse non-Euclidean data, 
yet methods using GCN \cite{gcn} 
 may fall short of model capacity, 
 and other methods adopting GAT \cite{gat} are mostly heuristic and lacking theoretical analysis. 
 A detailed review is provided in Sec. \ref{sec:litreview}.

In this paper, we propose a graph convolution model, called {\it L3Net},  
originating from low-rank graph filter decomposition, c.f. Fig. \ref{fig:diag1} (left).
The model provides a unified framework for graph convolutions, 
including  ChebNet \cite{chebnet}, GAT, EdgeNet \cite{isufi2020edgenets}
and CNN/geometrical CNN with low-rank filter as special cases.
In addition,
we theoretically prove that 
L3Net is strictly more expressive to represent graph signals than 
spectral graph convolutions based on global adjacency/graph Laplacian matrices, 
which is then empirically validated,
c.f. Sec. \ref{subsec:expressive}.
We also prove a Lipschitz-type representation stability of the new graph convolution layer using perturbation analysis.

\begin{figure}[tb]
\hspace{-12pt} 
\centering{
\includegraphics[trim={12pt 520pt 533pt 10pt},clip,height=.185\linewidth]{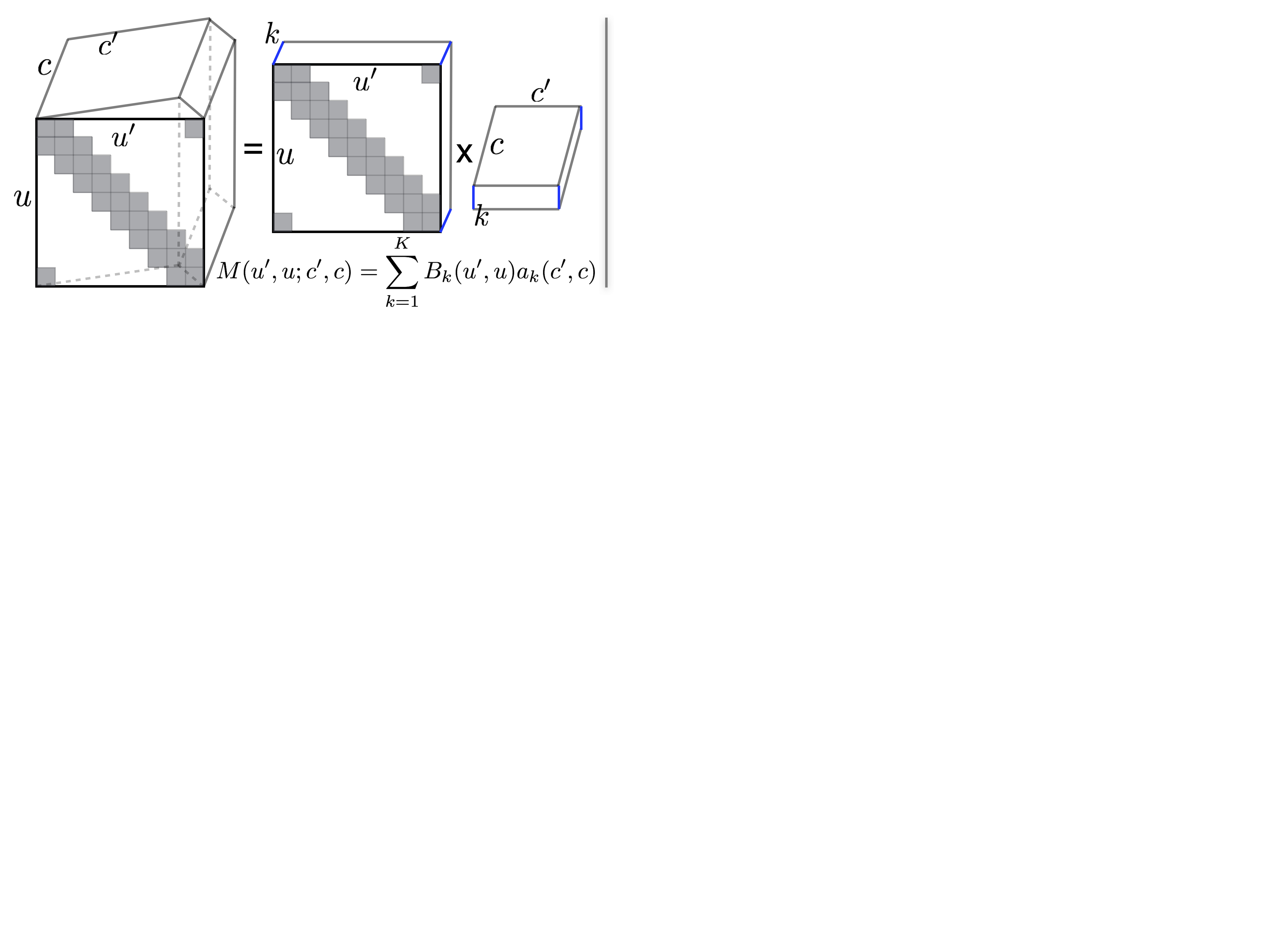} 
\includegraphics[trim={30pt 30pt 30pt 30pt},clip,height =0.16\textwidth]{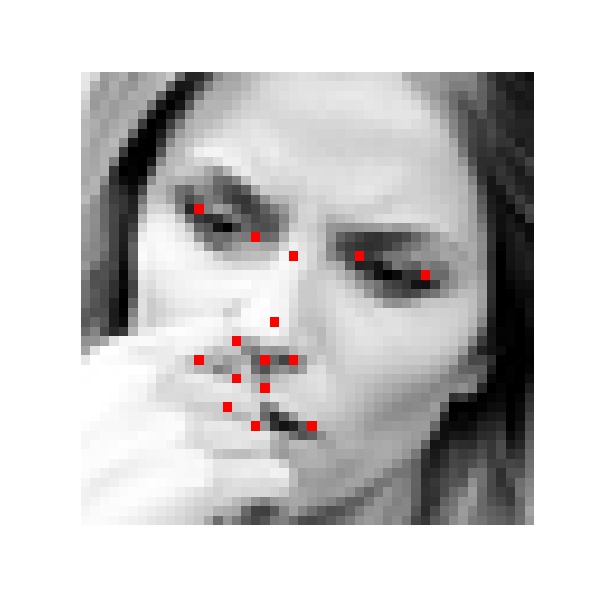}
\includegraphics[trim={50pt 30pt 50pt 30pt},clip,height =0.16\textwidth]{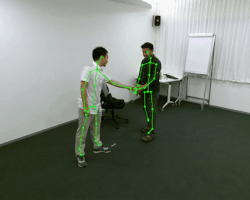}
\includegraphics[width=0.2\textwidth]{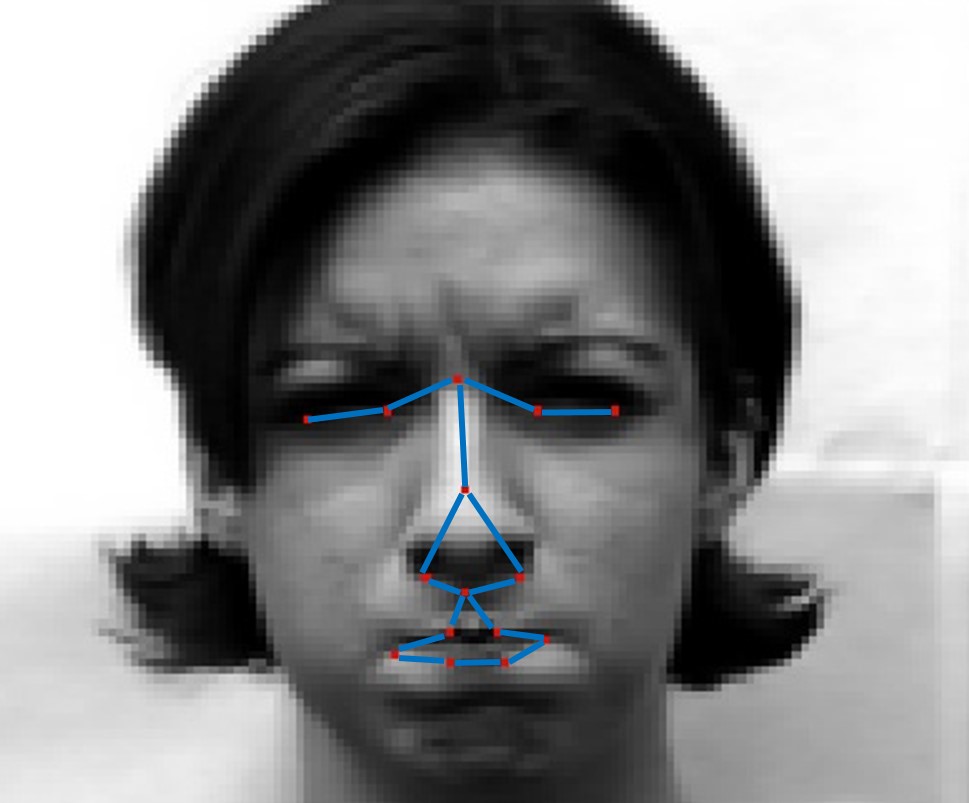}
}
\vspace{-3pt}
\caption{
(Left) $K$-rank graph local filters. $M$ is the tensor in the GNN linear mapping \eqref{eq:gnn} \eqref{eq:M-ours}, 
decomposed into learnable local basis $B_k$ combined by learnable coefficients $a_k$,
illustrated for the ring-graph on the right.
(Right) The first two figures shows the good property of landmarks for being invariant to pose and camera viewpoint changes. The third figure illustrates the graph we built on facial landmarks.
}
\label{fig:diag1}
\end{figure}

Because our model allows neighborhood specialized local graph filters, regularization may be needed to prevent over-fitting, so as to handle changing underlying graph topology and other graph noise, e.g., inaccurately detected landmarks or missing landmark points due to occlusions. Therefore, we also introduce a regularization scheme based on local graph Laplacians, motivated by the eigen property of the latter. This further
improves the representation stability aforementioned.
The improved performance of L3Net compared to other GNN benchmarks is demonstrated in a series of experiments,
and with the the proposed graph regularization, our model shows robustness to a variety of graph data noise.

In summary, the contributions  of the work are the following:
\begin{itemize}

    \item We propose a new graph convolution model 
        by a low-rank decomposition of graph filters over trainable local basis,
         which unifies several previous models of both spectral and spatial graph convolutions.

    \item Regularization by local graph Laplacians is introduced to improve the robustness against graph noise.
    
    \item We provide theoretical proof of the enlarged expressiveness for representing graph signals
    and the Lipschitz-type input-perturbation stability  of the new graph convolution model.
    
    \item We demonstrate with applications to object recognition of spherical data and facial expression/skeleton-based action recognitions using landmarks. Model robustness against graph data noise is validated on both real-world and simulated datasets.
\end{itemize}

\subsection{Related Works}\label{sec:litreview}

{\bf  Modeling on face/body landmark data.}
Many applications in computer vision, 
such as facial expression recognition (FER) and skeleton-based action recognition,
need to extract high-level features from landmarked data
which are sampled at irregular grid points on human face or at body joints. 
While CNN methods \cite{fer_cnn_guo2016deep, fer_cnn_ding2017, meng2017identity} 
prevail in FER task, landmark methods have the potential advantage in lighter model size
as well as more robustness to previously mentioned geometric transformations like pose variation.  
Earlier methods based on facial landmarks used hand-crafted features
\cite{fer_lmdk_jeong2018driver,fer_lmdk_morales2019use}
rather than deep networks. 
Skeleton-based methods in action recognition have been developed intensively recently \cite{survey_sbar},
including non-deep methods \cite{skar_hand_chellappa, skar_hand_wang2012mining} 
and deep methods \cite{skar_deep_ke2017new, skar_deep_kim2017interpretable, skar_deep_liu2016spatio, stgcn}. 
Facial and skeleton landmarks only give a coarse and irregular grid,
and then mesh-based geometrical CNN's are hardly applicable,
while previous GNN models on such tasks may lack sufficient expressive power.

{\bf Graph convolutional network.}
A systematic review can be found in several places, e.g. \cite{survey_gnn}.
Spectral graph convolution was proposed using full eigen decomposition of the graph Laplacian in \cite{bruna2013spectral},
Chebyshev polynomial in Chenbet \cite{chebnet}, 
by Cayley polynomials in \cite{cayleynet}.
 GCN \cite{gcn}, the mostly-used GNN,  is a variant of ChebNet using degree-1 polynomial.
 \cite{liao2019lanczosnet} accelerated the spectral computation by Lanczos algorithm.
Spatial graph convolution has been performed 
by summing up neighbor nodes' transformed features in NN4G \cite{nn4g}, 
by graph diffusion process in DCNN \cite{dcnn}, 
where the graph propagation across nodes is by the adjacency matrix.
Graph convolution with trainable filter has also been proposed in several settings:
MPNN \cite{mpnns} enhanced model expressiveness by message passing and sub-network; 
GraphSage \cite{graphsage} used trainable differential local aggregator functions in the form of LSTM or mean/max-pooling;
GAT \cite{gat} and variants \cite{agcn, gaan, genie} introduced  attention mechanism
to achieve adaptive graph affinity, which remains non-negative valued;
EdgeNet \cite{isufi2020edgenets} developed adaptive filters by taking products of trainable local filters.
Our model learns local filters which can take negative values and contains GAT and EdgeNet as special cases.
Theoretically, expressive power of GNN has been studied in \cite{morris2019weisfeiler,xu2019howpowerful,maron2019provably,maron2019invariant,keriven2019universal},
mainly focusing on distinguishing graph topologies, 
while our primary concern is to distinguish signals lying on a graph.

{\bf  CNN and geometrical CNN.} 
Standard CNN 
applies local filters translated and shared across locations on an Euclidean domain.
To extend CNN to  non-Euclidean domains,
convolution on a regular spherical mesh using geometrical information has been 
studied in S2CNN \cite{s2cnn}, SphereNet \cite{spherenet}, SphericalCNN \cite{sphericalcnn}, and UGSCNN \cite{ugscnn},
and applied to 3D object recognition,
for which other deep methods include 
3D convolutional \cite{modelnet_3dconv} and non-convolutional architectures \cite{modelnet_pointnet, modelnet_pointnet++}.
CNN's on manifolds construct weight-sharing across local atlas making use of a mesh,
e.g., by patch operator in \cite{masci2015geodesic},
anisotropic convolution in ACNN \cite{acnn},
mixture model parametrization in MoNet \cite{monet}, 
spline functions in SplineCNN \cite{splinecnn}, 
and 
manifold parallel transport in \cite{schonsheck2018parallel}. 
These geometric CNN models  use information of 
non-Euclidean meshes which usually need  sufficiently fine resolution.

\section{Method}

\subsection{Decomposed local filters}

Consider an undirected graph $G = (V,E)$, $|V|=n$.
A graph convolution layer maps from input node features $X(u',c')$ to output $Y(u,c)$, where $u, u' \in V$,
$c' \in [C']$ ($c \in [C]$) is the input (output) channel index,
the notation $[m]$ means $\{1, \cdots, m \}$, and
\begin{equation}\label{eq:gnn}
Y(u,c) = \sigma (  \sum_{u' \in V, c' \in [C]}  M(u',u; c',c) X(u',c') + \text{bias}(c) ),
\quad u\in V, \, c\in [C].
\end{equation}
The spatial and spectral graph convolutions correspond to different ways of specifying $M$, c.f. Sec. \ref{subsec:unified}.
The proposed graph convolution is defined as 
\begin{equation}\label{eq:M-ours}
M(u',u; c',c) = \sum_{k=1}^K a_k(c',c) B_k(u',u),
\quad a_k(c',c)  \in \R, 
\end{equation}
where $B_k(u',u)$ is non-zero only when $u' \in N_u^{(d_k)}$, 
$N_u^{(d)}$ denoting the $d$-th order neighborhood of $u$ (i.e., the set of $d$-neighbors of $u$),
and $K$ is a fixed number. 
In other words, $B_k$'s are $K$ basis of local filters around each $u$,
and the order $d_k$ can differ with  $1 \le k \le K$. 
Both $a_k$ and $B_k$ are trainable, so the number of parameters are 
$K \cdot C C' + \sum_{k=1}^K \sum_{u \in V} |N_u^{(d_k)}| \sim K \cdot CC' + K n p$,
where $p$ stands for the average local patch size. 
In our experiments we use $K$ up to $5$, and $d_k$ up to $3$. 
The construction \eqref{eq:M-ours} 
can be used as a layer type in larger GNN architectures.
Pooling of graphs can be added between layers,
and the choice of $K$ and neighborhood orders $(d_1,\cdots, d_K)$ can be adjusted accordingly.
The model may be extended in several ways
to be discussed in the last section.

\begin{figure}[t]
\hspace{-10pt}
    \begin{minipage}[h]{0.68\linewidth}
    \raisebox{5pt}{
	\includegraphics[trim={90pt 460pt 110pt 80pt}, clip, width=0.95\textwidth]{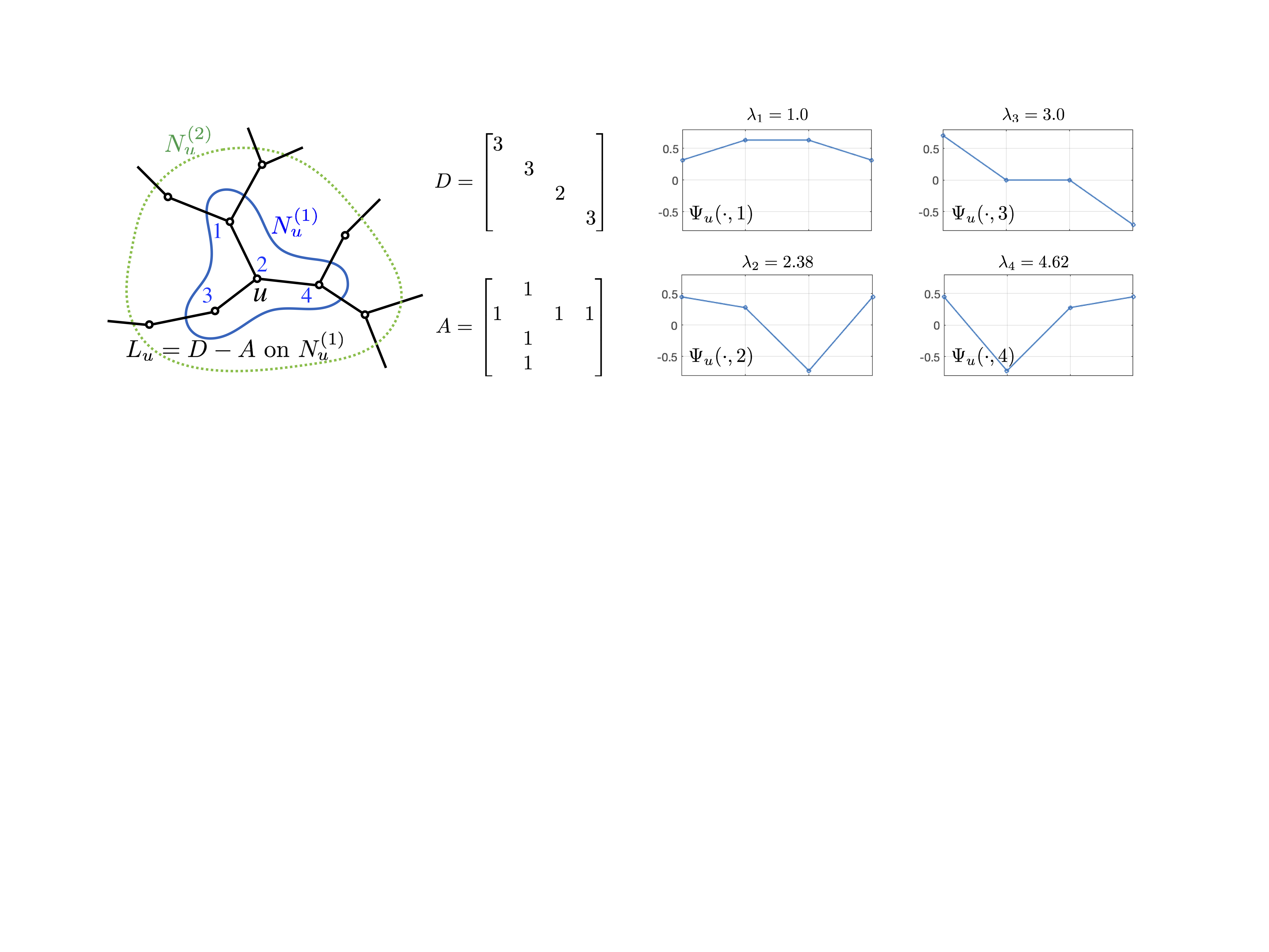}
	}
    \end{minipage}
\hspace{-10pt}   
    \begin{minipage}[h]{0.18\linewidth}
            \scriptsize
        \begin{tabular}{c|c}
        \hline
        Model &  \#params \\
        \hline
        ChebNet  / GCN  &  $L CC'$  / $CC'$ \\ 
        GAT                 &  $R (CC'+2C)$ \\
        EdgeNet     & $L(CC' + n p^{(1)})$\\
        Low-rank CNN   &  $K (CC'+p)$\\
        Locally-connected     & $CC' \cdot np$  \\
         \hline
        L3Net    &  $K( CC' + np) $\\
        \hline
        \end{tabular}
    \end{minipage}
\caption{
(Plots)
Local graph Laplacian $L_u := D-A$ on a neighborhood around node $u$.
The first Dirichlet eigenvector does not change sign on $N_u$  and is envelope-like.
(Table)
Model complexity measured by number of parameters,
$C$ and $C'$ being the number of input and output channels,
$p$ ($p^{(1)}$) the average patch size of local neighborhoods (local 1-neighborhoods),
see more in Sec. \ref{subsec:unified}.
}
\label{fig:diag3}
\end{figure}

\subsection{Regularization by local graph Laplacian}\label{sec:reg}

The proposed L3Net layer enlarges the model capacity by allowing $K$ basis filters at each location,
and a natural way to regularize the trainable filters is by the graph geometry, 
where, by construction, only the local graph patch is concerned. 
We introduce the following regularization penalty of the basis filters $B_k$'s as 
\begin{equation}\label{eq:def-reg}
{\cal R} ( \{B_k\}_k) =  
\sum_{k=1}^K 
\sum_{u \in V} (b_u^{(k)})^T L_u^{(k)} b_u^{(k)}, \quad b_u^{(k)}(v) : = B_k(v, u), \, b_u^{(k)}: N_u^{(d_k)} \to \R, 
\end{equation}
where $L_u^{(k)}$, equaling $(D-A)$  restricted to the subgraph on $N_u^{(d_k)}$, 
is the Dirichlet local graph Laplacian on $N_u^{(d_k)}$ \cite{chung1997spectral} (Fig. \ref{fig:diag3}).
The training objective is 
\begin{equation}\label{eq:train-obj}
{\cal L}( \{ a_k, B_k \}_k)  + \lambda {\cal R} ( \{B_k\}_k), \quad \lambda \ge 0,
\end{equation}
where ${\cal L}$ is the classification loss.
As ${\cal L}$ encourages the diversity of $B_k$'s,
the $K$-rankness usually remains a tight constraint in training,
unless $\lambda$ is very large,
see also  Proposition \ref{prop:large-reg}.

\subsection{A unified framework for graph convolutions}\label{subsec:unified}
Graph convolutions basically fall into two categories, the spatial and spectral constructions \cite{survey_gnn}. 
The proposed L3Net belongs to spatial construction,
and here we show that the model \eqref{eq:M-ours} is a unified framework for various graph convolutoins, both spatial and spectral.
Details and proofs are given in Appendix \ref{sec:proofs}.

$\bullet$
ChebNet \cite{chebnet}, GAT \cite{gat}, EdgeNet \cite{isufi2020edgenets}:
In ChebNet, $M$ per $(c',c)$ equals a degree-($L$-1) polynomial of the graph Laplacian matrix,
where the polynomial coefficients are trainable. 
GCN \cite{gcn} can be viewed as ChebNet with polynomial degree-$1$ and tied coefficients.
The attention mechanism in GAT enhances the model expressiveness by incorporating adaptive kernel-based non-negative affinities.
In EdgeNet,
the graph convolution operator is the product of trainable local filters supported on order-1 neighborhoods.
We have the following proposition:

\begin{proposition}\label{prop:larger-than-cheb}
L3Net \eqref{eq:M-ours} includes the following models as special cases:
\begin{itemize}
    \item[(1)] ChebNet (GCN) 
    when $K \ge L$ ($K \ge 2$), $L$ being the polynomial degree.
    \item[(2)] GAT when $K \ge R$, $R$ being the number of attention branches. 
    \item[(3)] EdgeNet when $K \ge L$, $L$ being the order of graph convolutions.
    \vspace{3pt}
\end{itemize}
\end{proposition}

$\bullet$
CNN: When nodes lie on a geometrical domain that allows translation ($u'-u$),
 in \eqref{eq:M-ours} setting $B_k(u',u) = b_k(u'-u)$ for some $b_k(\cdot)$ enforces spatial convolutional.
The convolutional kernel can be decomposed as $\sum_{k} a_k(c',c)b_k(\cdot)$ \cite{qiu2018dcfnet}.
Extension to CNN on manifold mesh is also possible as in  \cite{masci2015geodesic, splinecnn}.
We have the following:

\begin{proposition}\label{prop:larger-than-dcfcnn}
Mesh-based geometrical CNN's defined by linear patch operators,
including standard CNN on $\R^d$,
and with low-rank decomposed filters are special cases of  L3Net \eqref{eq:M-ours}.
\end{proposition}

We also note that
L3Net reduces from locally connected GNN \cite{coates2011selecting,bruna2013spectral},
the largest class of spatial GNN,
only by the low-rankness imposed by a small number of $K$ in \eqref{eq:M-ours}.
Locally connected GNN can be viewed as  \eqref{eq:gnn} with the requirement that for each $(c,c')$, 
$M(u', u; c', c)$ is nonzero only when $u'$ is locally connected in $u$.
The complexities of the various models are summarized  in Fig. \ref{fig:diag3} (Table),
where L3Net reduces from  the $np \cdot CC'$ complexity of locally-connected net to be the additive $(np + CC')$ times $K$.
When the number of channels $C$, $C'$ are large, e.g. in deep layers they $\sim 10^2$,
and the graph size is not large, 
e.g., in landmark data applications $np \ll CC'$,
the complexity is dominated by $K C C'$ which is comparable with ChebNet (GAT) if $K \approx L$ ($R$).
The computational cost is also comparable, 
as shown in experiments in Sec. \ref{sec:exp}. 
Furthermore, we have:

\begin{proposition}\label{prop:large-reg}
Suppose the subgraphs on $N_u^{(d_k)}$ are all connected, 
given $\alpha_{u,k} > 0 $ for all $u,k$, 
the minimum of \eqref{eq:def-reg} with constraint $\| b_u^{(k)} \|_2 \ge \alpha_{u,k}$ 
 is achieved  when $b_u^{(k)}$ equals the first Dirichlet eigenvector on $N_u^{(d_k)}$,
which does not change sign on $N_u^{(d_k)}$.
\end{proposition}

The proposition shows that in the strong regularization limit of $\lambda \to \infty$ in \eqref{eq:train-obj},
L3Net reduces to be ChebNet-like.
The constraint with constants $\alpha_{u,k}$ 
is included because
otherwise the minimizer will be $B_k$ all zero.
The first Dirichlet eigenvector is envelope-like (Fig. \ref{fig:diag3}), 
and then $B_k(\cdot, u)$ will be averaging operators on the local patch.
Thus the regularization parameter $\lambda$ can be viewed as trading-off between the more expressiveness in the learnable $B_k$,
and the more stability of the averaging local filters, similar to ChebNet and GCN.

\section{ Analysis }

We analyze the representation expressiveness 
and stability (defined in below)
of the proposed L3Net model.
All proofs in Appendix \ref{sec:proofs},
and experimental details in Appendix \ref{app:updown-detail}.

\subsection{Representation expressiveness of graph signals}
\label{subsec:expressive}

The theoretical question of graph signal representation expressiveness
concerns the ability for GNN deep features to distinguish graph signals. 
While related, the problem differs from the graph isomorphism test problem which has been intensively studied in the GNN expressiveness literature.
Here we prove that L3Net is strictly more expressive than certain spectral GNNs,
and support the theoretical prediction by experiments.

\begin{figure}
\centering
\scriptsize
    \begin{minipage}[h]{0.23\linewidth}
        \centering
        \includegraphics[height=70pt]{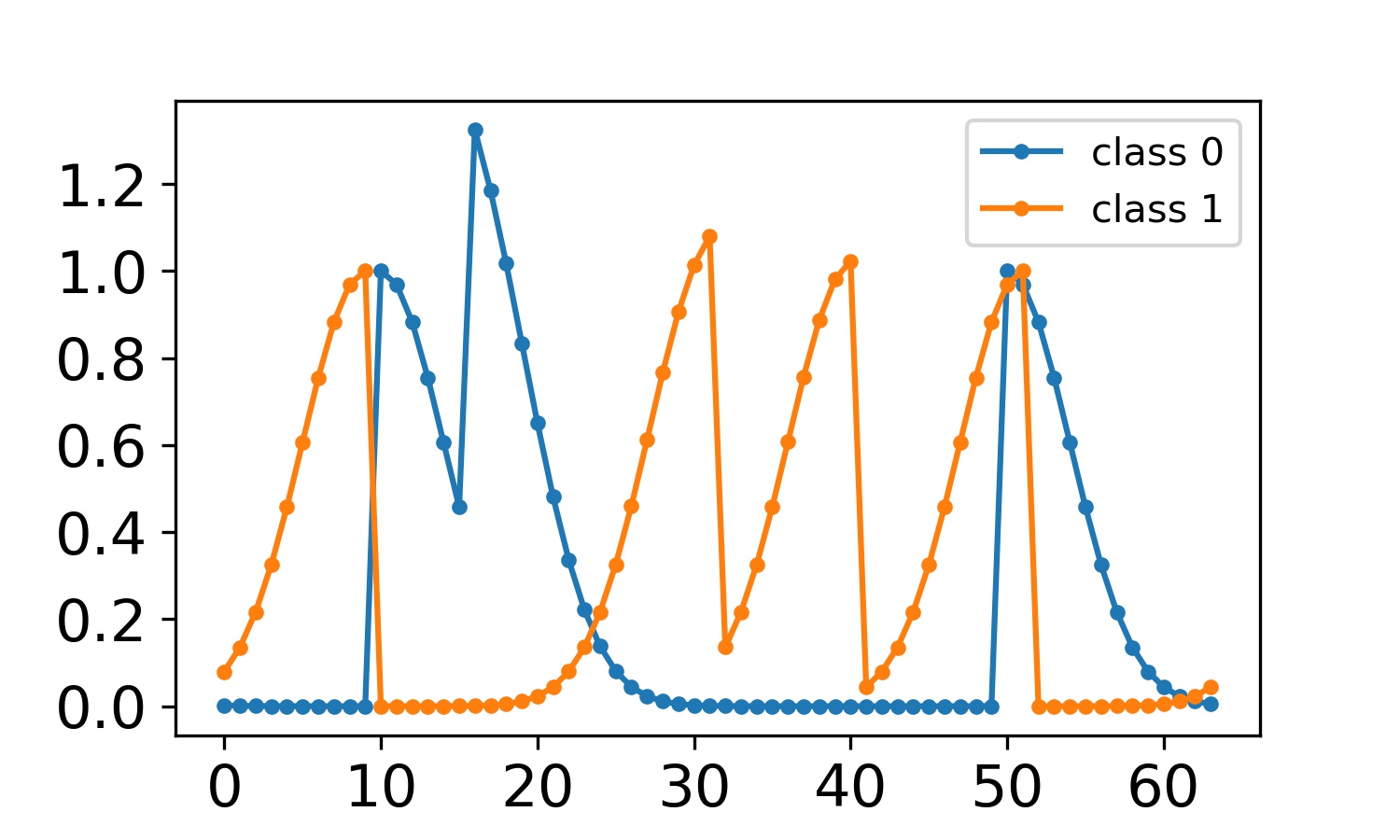}
    \end{minipage}
    \hspace{0.02\linewidth}
    \begin{minipage}[h]{0.5\linewidth}
        \centering
        \footnotesize
        \begin{adjustbox}{max width=\textwidth}
        \begin{tabular}{c|c|c|c||c}
        \hline
        Model & order & \#params & ring graph Acc & chain graph Acc\\
        \hline
        \multirow{4}*{ChebNet}  & L=3 & 6.5k & $51.71\pm0.24$ & $51.05\pm0.33$ \\
        ~                       & L=5 & 10.7k & $51.62\pm0.24$ & $51.07\pm0.37$ \\
        ~                       & L=30 & 62.7k & $51.32\pm0.38$ & $51.01\pm0.41$ \\
        \hline
        GAT (R=1)    & 1 & 1.3k  & $51.62\pm0.14$ & $51.46\pm0.94$ \\
        GAT (R=8)    & 1 & 10.4k  & $57.82\pm8.06$ & $58.04\pm9.13$ \\
        \hline
        WLN    & 1 & 4.5k  & $50.99\pm0.36$ & $50.8\pm0.08$  \\
        \hline
        MPNN         & 1 & 9.4k  & $51.06\pm0.32$ & $50.94\pm0.09$ \\
        \hline
        \multirow{3}{*}{L3Net}  & 1 & 2.7k  & $99.82\pm0.05$ & $99.69\pm0.09$ \\
        ~                           & 0;1;2 & 7.4k & $99.93\pm0.03$ & $99.85\pm0.04$ \\
        ~                       & $1^*$ & 2.3k & $\mathbf{99.96\pm0.01}$ & $\mathbf{99.94\pm0.01}$\\ 
        \hline
        \end{tabular}
        \end{adjustbox}
    \end{minipage}
    \vspace{0.08\textwidth}
    \begin{minipage}[h]{0.23\textwidth}
        \centering
        \includegraphics[width=90pt,height=65pt]{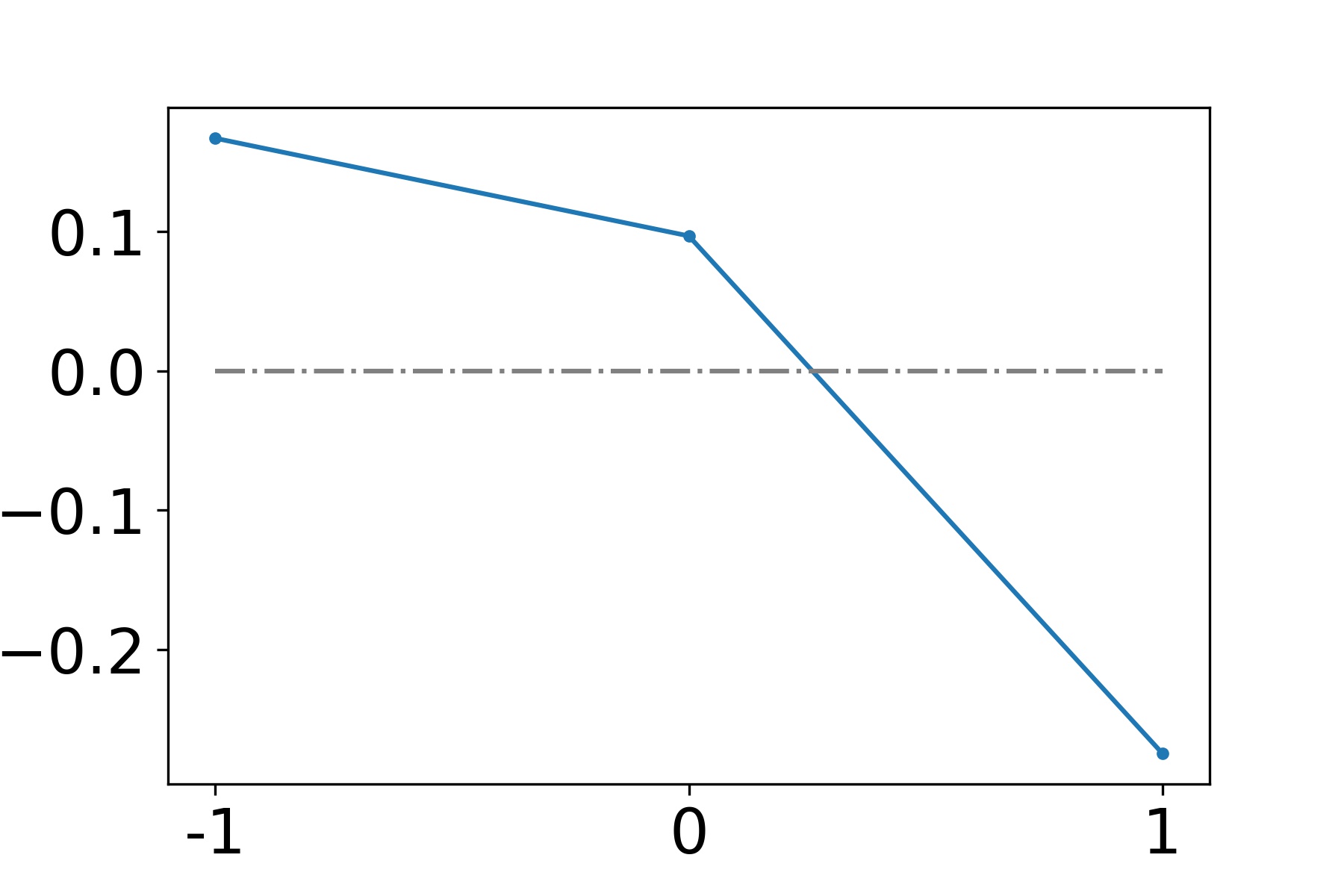}
    \end{minipage}
    \vspace{-35pt}
    \caption{\small
     Up/down-wind classification.
     (Plots) Left: example data from two classes. Right: learned shared basis on the graph neighborhood of 3, corresponding to the last row in the table. 
      (Table) Test accuracy by MPNN \cite{mpnns}, WLN \cite{morris2019weisfeiler}, ChebNet up to $L$=30 and L3Net $K$=1 and 3, as well as GAT with different heads.
      Last row order 1 with star: L3Net with shared  basis $B(\cdot, u)$ across all locations $u$. 
     }
   \label{fig:updown}
\end{figure}

We have shown that the L3Net model contains ChebNet (Proposition \ref{prop:larger-than-cheb}),
and the following proposition proves the strictly more expressiveness for graph signal classification.
We call $B$ a graph local filter if $B(u,v)$ is non-zero only when $v$ is in the neighborhood of $u$.
In a spectral GNN, the graph convolution takes the form as $x \mapsto f(A) x$ where $f$ is a function on $\R$, 
and $A$ is the (possibly normalized) adjacency matrix.

\begin{proposition}\label{prop:expressive}
There is a graph and
1) A local filter $B$ on it such that
$B$ cannot be expressed by any spectral graph convolution, but can be expressed by L3Net with $K=1$.
2)  Two data distributions on the graph (two classes)
such that,
with a group invariant operator in the last layer,
the deep feature of any spectral GNN cannot distinguish the two classes, 
but that of L3Net with 1 layer and $K=1$ can.
\end{proposition}

The fundamental argument is that spectral GNN is permutation equivariant
(see e.g. \cite{gama2019stability}, reproduced as Lemma \ref{lemma:node-permutation}),
and the local filters in L3Net break such symmetry to obtain more discriminative power.
The constructive example used in the proof is on a ring graph (Fig. \ref{fig:ring_expressive}, $A$ and the basis $B$),
and the two data distributions shown  in Fig. \ref{fig:updown}.
Proposition \ref{prop:expressive} gives that,
on the ring graph and using GNN with a global pooling in the last layer,
an L3Net layer with $K=1$ can have classification power while a ChebNet with any order cannot. 
On a chain graph (removing the connection between two end points in a ring graph), 
which not exactly follows the theory assumption,
since the two graphs only differ at one edge, we expect that it will remain a difficult case for the ChebNet but not for L3Net.
To verify the theory, we conduct experiments 
using a two-layer GNN and the results are in Fig. \ref{fig:updown} (table). 
In the last row, we further impose shared basis across nodes 
which reduces L3Net to a 1D convolutional layer,
and the learned basis shows a ``difference'' shape (right plot) which explains its classification power.
Results are similar using a 1-layer GNN (Tab. \ref{tab:up-down-gcn-1}).
The argument in Proposition \ref{prop:expressive} extends to other graphs and network types.
Generally, when a GNN based on global graph adjacency or Laplacian matrix
applies linear combinations of local averaging filters,  
then certain graph filters may be difficult to express.
We experimentally examine GAT, WLN and MPNN, which underperform on the binary classification task,
as shown in Fig. \ref{fig:updown} (table).

\subsection{Representation stability}\label{subsec:analysis-stable}

We derive perturbation bounds of GNN feature representation,
which is important for robustness against data noise. 
The analysis implies a trade-off between de-noising and keeping  high-frequency information,
which is consistent with experimental observation in Sec. \ref{sec:exp}.

Consider the change in the GNN layer output $ Y$ defined in \eqref{eq:gnn}\eqref{eq:M-ours}
when the input $X$ changes. 
For simplicity, let $C=C'=1$, and the argument extends.  
For any graph signal $x: V \to \R$ and $ V' \subset V$, define 
$\|x \|_{2, V'} := (\sum_{u \in V'} x(u)^2)^{1/2}$ and $\langle x, y\rangle_{V'} = \sum_{u \in V'} x(u)y(u)$.
The following perturbation bound holds for the L3Net layer with/without regularization.

\begin{theorem}\label{thm:stable-1}
Suppose that $X = \{ X(u) \}_{u \in V}$ is perturbed to be $\tilde{X} = X + \Delta X $, 
 the activation function $\sigma: \R \to \R$ is non-expansive,
 and $\sup_{u \in V} \sum_{k=1}^K |N_u^{(d_k)} | \le K p$,
then the change in the output $\{ Y(u) \}_{u \in V}$ in $2$-norm is bounded by
\[
\| \Delta Y \|_{2,V} \le  \beta^{(1)} \cdot \|a\|_2 \sqrt{Kp} \|\Delta X\|_{2,V},
\quad
\beta^{(1)} := \sup_{k, u} \| B_k( \cdot, u)\|_{2, N_u^{(d_k)}}.
\]
\end{theorem}

Note that $p$ indicates the averaged size of the $d_k$-order local neighborhoods.
The proposition implies that when $K$ is $O(1)$,
and the local basis $B_k$'s have $O(1)$ 2-norms on all local parches uniformly bounded by $\beta^{(1)}$,
then the Lipschitz constant of the GNN layer mapping is $O(1)$,
i.e., the product of $\|a\|_2$, $\beta^{(1)}$ and $\sqrt{Kp}$,
which does not scale with $n$.
This resembles the generalizes the 2-norm of a convolutional operator which only involves the norm of the convolutional kernel,
which is possible due to the local receptive fields in the spatial construction of L3Net.

The local graph regularization introduced  in Sec. \ref{sec:reg}
improves the stability of $Y$ w.r.t. $\Delta X$
by suppressing the response 
to local high-frequency perturbations in $\Delta X$.
Specifically, the local 
graph Laplacian $L_u^{(k)}$ on the subgraph on $N_u^{(d_k)}$
is positive definite whenever the subgraph is connected and not isolated from the whole graph.
We then define the weighted 2-norm on local patch
$\|x\|_{L_u^{(k)}} := \langle x, L_u^{(k)} x \rangle_{N_u^{(d_k)}}$,
and similarly $\|x\|_{(L_u^{(k)})^{-1}}$.

\begin{theorem}\label{thm:stable-2}
Notation and setting as in Theorem \ref{thm:stable-1},
if furtherly, all the subgraphs on $N_u^{(d_k)}$ are connected within itself and to the rest of the graph,
and  there is $ \rho  \ge 0$ s.t. 
\[
\| \Delta X \|_{(L_u^{(k)})^{-1}} \le \rho \| \Delta X \|_{2, N_u^{(d_k)}}, \quad \forall u, k,
\]
then
\[
\| \Delta Y \|_{2,V} \le  \rho \beta^{(2)}  \cdot \|a\|_2 \sqrt{Kp}  \| \Delta X \|_{2,V},
\quad 
\beta^{(2)}: = \sup_{ k,u} \| B_k(\cdot, u) \|_{L_u^{(k)}}.
\]
\end{theorem}

The bound improves from Theorem \ref{thm:stable-1} when $\rho \beta^{(2)} <  \beta^{(1)}$,
and regularizing by ${\cal R} = \sum_{u,k}\| B_k(\cdot, u) \|_{L_u^{(k)}}^2 $ leads to smaller $\beta^{(2)}$.
Meanwhile, on each $N_u^{(d_k)}$ the Dirichlet eigenvalues increases
$ 0 < \lambda_1 \le \lambda_2 \cdots \le \lambda_{p_{u,k}} $, $p_{u,k}:= |N_u^{(d_k)}|$, 
 thus weighting by $\lambda_l^{-1}$ in $\| \cdot \|_{(L_u^{(k)})^{-1}} $ 
decreases the contribution from high-frequency eigenvectors. 
As a result, $\rho$ will be small if $\Delta X$ contains a significant high-frequency component on the local patch,
e.g., additive Gaussian noise or missing values.
Note that in the weighted $2$-norm of $\Delta X$ by $( L_u^{(k)} )^{-1}$,
only the relative amount of high-frequency component 
in $\Delta X$ matters (because any constant normalization of $L_u^{(k)}$ cancels in the product of $\rho$ and $\beta^{(2)}$).
The benefits of local graph regularization in presence of noise in graph data will be shown in experiments.

\begin{figure}[t]
    \vspace{-10pt}
    \small
    \begin{minipage}[h]{0.1\textwidth}
        \centering
        \includegraphics[width=\textwidth]{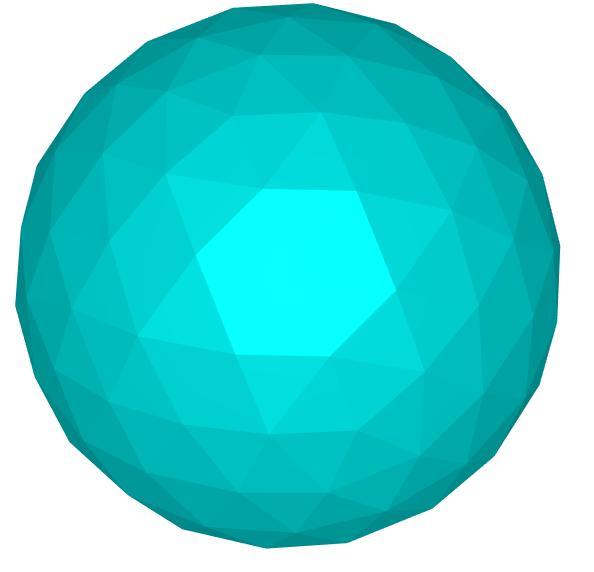}
        \vspace{-15pt}
        \caption*{\small \bf level 2}
    \end{minipage}
    \hspace{0.01\textwidth}
    \begin{minipage}[h]{0.1\textwidth}
         \centering
         \includegraphics[width=0.95\textwidth]{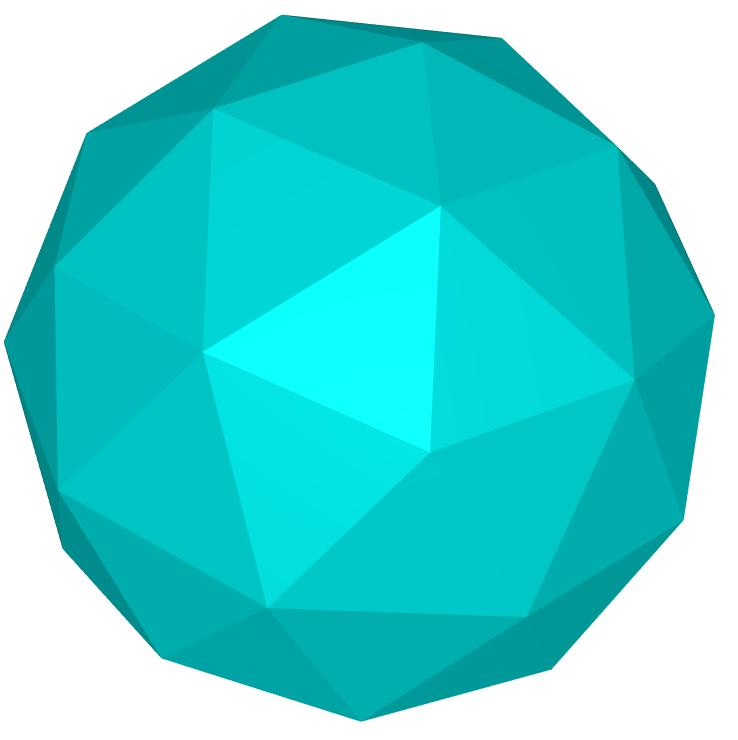}
        \vspace{-15pt}
        \caption*{\small \bf level 1}
    \end{minipage}
    \hspace{0.01\textwidth}
    \begin{minipage}[h]{0.73\textwidth}
        \centering
        \begin{adjustbox}{max width=\textwidth}
        \begin{tabular}{c|c|c|c|c|c|c||c|c} 
        \hline 
        Model & \tabincell{c}{4;3;2\\Acc} & \tabincell{c}{3;2;1\\Acc} & \tabincell{c}{3;2;0\\Acc} & \tabincell{c}{3;1;0\\Acc} & \tabincell{c}{2;2;1\\Acc} & \tabincell{c}{2;1;0\\Acc} & \tabincell{c}{\bf{3;0;0}\\\bf{Acc}} & \tabincell{c}{\bf{2;0;0}\\\bf{Acc}}\\
        \hline  
        \hline
        UGSCNN             & 99.2 & 98.81 & 97.52 & 97.96 & \bf{98.22} & 97.77 & 75.75 & 86.61\\
        \hline
        GCN                  & 95.8 & 90.46 & 75.62 & 84.31 & 94.01 & 83.24 & 27.92 & 37.07\\
        ChebNet        & \bf{99.3}  & 98.50 & 98.07 & 97.07 & 97.12 & 95.51 & 73.1 & 90.73\\
        \hline
        {\bf L3Net} (1;1;2;3)            & 99.1 & \bf{98.81} & \bf{98.89} & \bf{98.60} & 97.76 & \bf{97.97} & \bf{93.14} & \bf{97.26}\\
        \hline
        \end{tabular}
        \end{adjustbox}
    \end{minipage}
    
    \vspace{-5pt}
    \caption{
     \small
    (Plot) Icosahedral spherical meshes at level 2 and 1.
    (Table) Testing accuracies of sphere MNIST under different mesh settings,
	($l1;l2;l3$) stands for the mesh level used in each GNN layer.
	L3Net uses $K$=4, and neighborhood order (1;1;2;3). 
	 S2CNN \cite{s2cnn} on mesh (4;3;2) has accuracy 96.0.
	}
    \label{fig:spheremnsit_exp}
\end{figure}

\section{Experiment}\label{sec:exp}

We test the proposed L3Net model on several datasets.\footnote{Code link: \href{https://github.com/ZichenMiao/L3Net}{https://github.com/ZichenMiao/L3Net}}

\subsection{Object recognition of data on spherical mesh}

We first classify data on a spherical mesh:
sphere MNIST and sphere ModelNet-40,
following the settings in literature.
Though regular mesh on sphere is not the primary application scenario that motivates our model, 
we include the experiments to compare with benchmarks and test the efficiency of L3Net on such regular meshes. 
Following UGSCNN \cite{ugscnn}, we implement different mesh resolution on a sphere, indicated by ``mesh level'' (Fig. \ref{fig:spheremnsit_exp}), where number of nodes in different levels can vary from 2562 (level 4) to 12 (level 0).
All the networks consist of three convolutional layers,
see more details in Appendix \ref{app:detail-1}. 
Using the original mesh level (4;3;2), the finest resolution as in UGSCNN,
L3Net gives among the best accuracies for sphere MNIST. 
On Modelnet-40, 
L3Net achieves a testing accuracy of 90.24, 
outperforming ChebNet and GCN and and is comparable to UGSCNN
which uses spherical mesh information (Tab. \ref{tab:spheremnist-fine}). 
When the mesh becomes coarser, 
as shown in Fig. \ref{fig:spheremnsit_exp} (Table),
L3Net improves over GCN and ChebNet ($L$=4) and is comparable with UGSCNN under nearly all mesh settings. 
We observe that in some settings ChebNet can benefit from larger $L$, but the overall accuracy is still inferior to L3Net.
The most right two columns give two cases of coarse meshes where L3Net shows the most significant advantage.

\begin{table}[tb]
\scriptsize
\centering
\caption{
    Results on CK+ and FER13, with comparison to $\text{CNN}^\dag$\cite{fer_cnn_ding2017}, $\text{CNN}^\ddag$ \cite{fer_cnn_guo2016deep}, landmark method using handcrafted features \cite{fer_lmdk_morales2019use}, and various GNN methods. Specifically, we compare to GAT \cite{gat} with different \#heads (h) and \#features (f). The mean testing time on CK+:
	ChebNet ($L$=4) 12.56ms, L3Net (order 1,1,2,3) 13.02ms. GAT (h=f=8) 39.67ms, (h=f=16) 41.02ms.}
\begin{tabular}{cc|c|c|c|c}
\cline{3-6} 
      &     & \multicolumn{2}{c|}{CK+}    & \multicolumn{2}{c}{FER13} \\
\hline 
\multicolumn{1}{c|}{Model} & \tabincell{c}{Bases\\Order} & \tabincell{c}{\#params\\(w/o FC)} & Acc & \tabincell{c}{\#params\\(w/o FC)} & Acc \\
\hline
\hline
\multicolumn{1}{c|}{ $\text{CNN}^\dag$ } & - & 7M & 98.60 & - & -\\
\multicolumn{1}{c|}{$\text{CNN}^\ddag$.} & - & - & - & 2.6M & 71.33 \\
\multicolumn{1}{c|}{Landmarks-handcraft} & - & - & $91.00\pm0.03$ & - & - \\
\hline
\multicolumn{1}{c|}{GAT (h=8, f=8)}   & 1   & 34.6k  & $91.62\pm1.16$ & 46.9k & 49.50 \\
\multicolumn{1}{c|}{GAT (h=16, f=16)} & 1   & 142.3k & $90.87\pm0.78$ & 151.1k & 48.93 \\
\multicolumn{1}{c|}{GCN }             & 1   & 34.5k  & $91.78\pm0.38$ & 42.6k & 55.54 \\
\multicolumn{1}{c|}{GraphConv}        & 1   & 169.6k & $81.62\pm0.48$ & 215.4k & 55.63 \\
\hline
\multicolumn{1}{c|}{\multirow{3}*{ChebNet }} & $L$=3 & 102.3k & $92.93\pm0.59$ & 136.4k & 59.68 \\
\multicolumn{1}{c|}{~} & $L$=4     & 136.3k & $93.22\pm0.37$ & 181.6k & 60.26 \\
\multicolumn{1}{c|}{~} & $L$=5   & 170.2k & $93.03\pm0.62$ & 227.3k & 60.29 \\
\hline
\multicolumn{1}{c|}{\multirow{2}*{EdgeNet }} & $L$=3 & 103.4k & $92.41\pm0.81$ & 137.2k & 58.73 \\
\multicolumn{1}{c|}{~}              & $L$=4     & 137.1k & $92.57\pm0.84$ & 182.5k & 60.05 \\
\hline
\multicolumn{1}{c|}{\multirow{6}*{\bf{L3Net}}} & 2;2;2  & 102.8k & $95.32\pm0.31$   & 139.7k & 60.46 \\    
\multicolumn{1}{c|}{~}      & 0;1;2;3  & 136.8k & $95.03\pm0.30$   & 182.8k & 60.65 \\
\cline{2-6}
\multicolumn{1}{c|}{~}      & 1;1;2  & \multirow{2}*{102.7k} & $94.68\pm0.56$ & \multirow{2}*{139.4k} & 59.68 \\
\multicolumn{1}{c|}{~}      & +reg${0.005}$ & ~ & $94.52\pm0.61$  & ~ & 61.13 \\
\cline{2-6}
\multicolumn{1}{c|}{~}      & 1;1;2;3  & \multirow{2}*{136.9k} & $\mathbf{95.37\pm0.60}$  & \multirow{2}*{183.0k} & 60.71  \\
\multicolumn{1}{c|}{~}      & +reg${0.5}$ & ~ & $95.11\pm0.44$ & ~ & \bf{61.64}  \\
\hline
\end{tabular}
\label{tab:fer_results}
\end{table}

\subsection{Facial expression recognition (FER)}

We test on two FER datasets,
Extended CohnKanade (CK+) \cite{ck+_dataset} and FER13 \cite{fer13_dataset}.  
We use 15 facial landmarks, see Fig. \ref{fig:diag1},
and pixel values on a patch around each landmark point as node features.
Details about dataset and model setup are in Appendix \ref{app:detail-2}.
Unlike spherical mesh, facial and body landmarks (next section) are coarse irregular grids where no clear pre-defined mesh operation is applicable. 
We benchmark L3Net with other GNN approaches, as shown in Table \ref{tab:fer_results}.
The local graph regularization strategy is applied on FER13, due to the severe outlier data of landmark detection caused by occlusion. 
On CK+, L3Net leads all non-CNN models by a large margin,
and the best model (1,1,2,3) uses comparable number of parameters with the best ChebNet ($L$=4). 
On FER13, L3Net has lower performance than ChebNet and EdgeNet \cite{isufi2020edgenets}, but outperforms after adding regularization. 
The running times of best ChebNet and L3Net models are comparable, and are much less than GAT's.

\subsection{Action recognition}

We test on two skeleton-based action recognition datasets, 
NTU-RGB+D \cite{shahroudy2016ntu} and Kinetics-Motion \cite{kay2017kinetics}.
The irregular mesh is the 18/25-point body landmarks, 
with graph edges defined by body joints, shown in Fig. \ref{fig:diag1} and Fig. \ref{fig:body_Vitruvian_Man}. 
We adopt ST-GCN \cite{stgcn} as the base architecture,
and substitute the GCN layer 
with new L3Net layer, called \textbf{ST-L3Net}. 
On Kinetics-Motion, we adopt the regularization mechanism to overcome the severe data missing caused by camera out-of-view. 
See more experimental details in Appendix \ref{app:detail-3}.
We benchmark performance with ST-GCN \cite{stgcn}, ST-GCN (our implementation without using geometric information) and ST-ChebNet (replacing GCN with ChebNet layer),
shown in Table \ref{tab:action}.
L3Net shows significant advantages on two NTU tasks, cross-view and cross-subject settings.
On Kinetics-Motion,  L3Net regains superiority over other models after applying regularization.
The results in both Table \ref{tab:fer_results} and \ref{tab:action} 
indicate that stronger regularization sacrifices expressiveness for clean data and gains stability for noisy data, 
which is consistent with the theory in Sec. \ref{subsec:analysis-stable}.

\begin{table}[th]
\scriptsize
\centering
\caption{
Results on NTU-RGB+D and Kinetics-Motion}
\begin{tabular}{cc|c|c|c|c|c} 
\cline{3-7}
  &  &  \multicolumn{3}{c|}{NTU-RGB+D} & \multicolumn{2}{c}{Kinetics-Motion} \\ 
\hline
\multicolumn{1}{c|}{Model} & \tabincell{c}{Bases\\order} & \tabincell{c}{\#params\\(w/o FC)} & x-view Acc & x-sub Acc & \tabincell{c}{\#params\\(w/o FC)} & Acc\\ 
\hline
\hline
\multicolumn{1}{c|}{ST-GCN \cite{stgcn}} & 1 & - & 88.30 & 81.50 & - & 72.4 \\
\hline
\multicolumn{1}{c|}{ST-GCN} & 1 & 2.6M & 82.59 & 74.33 & 1.4M & 72.85 \\
\hline
\multicolumn{1}{c|}{\multirow{3}*{ST-ChebNet }} & $L$=3 & 3.1M & 86.40 & 78.24 & 1.8M & 77.91 \\
\multicolumn{1}{c|}{~}                  & $L$=4   & 3.3M & 86.45 & 80.20 & 2.1M & 78.24 \\
\multicolumn{1}{c|}{~}                  & $L$=5 & 3.5M & 76.70 & 71.42 & 2.3M & 77.57 \\
\hline
\multicolumn{1}{c|}{\multirow{4}*{\bf{ST-L3Net}}} & 1;1;2 & \multirow{2}*{3.1M} & 90.78 & \bf{83.64} & \multirow{2}*{1.8M} & 75.20 \\
\multicolumn{1}{c|}{~}                  & +reg${0.01}$ & ~ & 88.38 & 81.54 & ~ & \bf{78.49} \\
\cline{2-7}
\multicolumn{1}{c|}{~}                  & 1;1;2;3    & \multirow{2}*{3.3M} & \bf{91.52} & 82.46 & \multirow{2}*{2.1M} & 75.07 \\
\multicolumn{1}{c|}{~}                  & +reg${0.01}$ & ~ & 89.87 & 80.97 & ~ & 76.68 \\
\hline
\end{tabular}
\label{tab:action}
\end{table}

\subsection{Robustness to graph noise}

To examine the robustness to graph noise,
we experiment on down-sampled MNIST data on 2D regular grid  with 4-nearest-neighbor graph.
With no noise,
on 28$\times$28 data (Tab.  \ref{tab:mnist-28}),
14$\times$14 data (Tab.  \ref{tab:mnist-14}),
and 7$\times$7 data (Tab. \ref{tab:mnist7_gaussian} ``original'' column), 
the performance of L3Net is comparable to ChebNet \cite{chebnet} and EdgeNet \cite{isufi2020edgenets} and better than other GNN methods.
We consider three types of noise,
Gaussian noise added to the pixel value,
missing nodes or equivalently missing value in image input,
and permutation of the node indices,
details in Appendix \ref{app:detail-4}. 
The results of adding different levels of gaussian noise and permutation noise are shown in Tab. \ref{tab:mnist7_gaussian}, while results of adding missing value noise is provided in Appendix \ref{app:detail-4}. The results show that our regularization scheme improves the robustness to all three types of graph noise,
supporting the theory in Sec. \ref{subsec:analysis-stable}.
Specifically,
L3Net without regularization may underperform than ChebNet, 
but catches up after adding regularization, which is consistent with Proposition \ref{prop:large-reg}.

\begin{table}[h]
\scriptsize
\centering
\caption{
Results on MNSIT with grid size $7\times7$ with different levels of Gaussian noise and Permutation noise.
}
\hspace{-8pt}
\begin{adjustbox}{max width=\textwidth}
\begin{tabular}{c|c|c|c||c|c|c||c} 
\hline
Model & \tabincell{c}{bases\\order} & \tabincell{c}{\#params\\(w/o FC)} & Acc(original) & \tabincell{c}{Acc (gaussian)\\(psnr 24.9)} & \tabincell{c}{Acc (gaussian)\\(psnr 19.1)} & \tabincell{c}{Acc (gaussian)\\(psnr 15.7)} & \tabincell{c}{Acc\\(permutation)}\\
\hline
\hline
GCN           & 1       & 2.4k  & $90.02\pm0.24$ & $89.27\pm0.09$ & $85.70\pm0.13$ & $81.32\pm0.18$  & $83.00\pm0.18$ \\
\hline
\multirow{3}*{ChebNet}  & $L$=3   & 6.5k  & $92.85\pm0.09$ & $91.13\pm0.15$ & $87.64\pm0.23$ & $82.70\pm0.33$ & $86.94\pm0.06$\\
~                       & $L$=5  & 10.7k & $93.2\pm0.07$ & $91.92\pm0.11$ & $88.22\pm0.10$ & $83.04\pm0.12$ & $87.27\pm0.23$\\
~                       & $L$=7  & 14.8k & $93.45\pm0.06$ & $91.80\pm0.10$ & $87.84\pm0.15$ & $83.75\pm0.14$ & $87.53\pm0.19$\\
\hline
GAT (h=8,f=16) & 1     & 17.5k  & $79.50\pm1.24$ & $68.68\pm0.45$ & $64.8\pm1.69$ & $65.38\pm1.03$ & $62.21\pm0.56$ \\
\hline
MPNN            & 1     & 18.8k  & $86.94\pm0.37$ & $85.36\pm0.51$ & $82.23\pm0.35$ & $77.59\pm0.34$ & $77.55\pm0.26$ \\
\hline
WLN       & 1     & 17.1k  & $87.61\pm0.04$ & $86.01\pm0.20$ & $83.60\pm0.09$ & $79.47\pm0.11$ & $80.51\pm0.05$ \\
\hline
\multirow{2}*{EdgeNet}  & $L$=3  & 7.5k & $93.26\pm0.16$ & $91.81\pm0.14$ & $88.42\pm0.36$ & $84.56\pm0.40$ & $87.15\pm0.30$ \\
~                       & $L$=4  & 10.1k & $93.44\pm0.17$ & $92.27\pm0.16$ & $88.60\pm0.17$ & $84.15\pm0.59$ & $87.44\pm0.28$ \\
\hline
\multirow{4}*{\bf{L3Net}} & 0;1;2 & 8.1k & $93.45\pm0.10$ & - & - & - & -\\
\cline{2-8}
~                       & 1;1;2 & \multirow{2}*{8.4k}  & $93.56\pm0.08$ & $92.10\pm0.08$ & $88.20\pm0.13$ & $83.00\pm0.33$ & $87.58\pm0.19$ \\
~                       & +reg${0.5}$ & ~  & $ \mathbf{93.85\pm0.13} $ & $92.31\pm0.07$ & $\mathbf{89.23\pm0.10}$ & $84.59\pm0.23$ & $88.08\pm0.18$ \\
\cline{2-8}
~                       & 1;1;2;3  & \multirow{2}*{12.2k}  & $93.67\pm0.15$ & $92.25\pm0.15$ & $88.28\pm0.16$ & $82.80\pm0.37$ & $87.66\pm0.12$ \\
~                       & +reg${0.5}$ & ~ & ${93.85\pm0.15}$ & $\mathbf{92.56\pm0.12}$ & $89.15\pm0.24$ & $\mathbf{84.61\pm0.25}$ & $\mathbf{88.21\pm0.15}$ \\
\hline
\end{tabular}
\end{adjustbox}
\label{tab:mnist7_gaussian}
\end{table}

\section{Conclusion and Discussion}

The paper proposes a new graph convolution model using learnable local filters decomposed over a small number of basis.
Strengths:
Provable enhancement of model expressiveness with significantly reduced model complexity from locally connected GNN.
Improved stability and robustness via local graph regularization, supported by theory. 
Plug-and-play layer type,
suitable for GNN graph signal classification problems on relatively unchanging small underlying graphs, 
like face/body landmark data in FER and action recognition applications. 

Limitations and extensions:
(1) Scalability to larger graph. When $|V|=n$ is large, 
the complexity increase in the $npK$ term would be significant.
The issue in practice can be remedied by mixing use of layer types,
e.g., only adopting L3Net layers in upper levels of mesh which are of reduced size.
(2) Dynamically changing underlying graph across samples.
For more severe changes of the underlying graph, 
we can benefit from solutions such as node registration or other preprocessing techniques,
possibly by another neural network.
(3) Incorporation of edge features. 
Edge features can be transformed into 
extra channels of node features by an additional layer in the bottom, 
and the low-rank graph operation can be similarly employed there. 

\section*{Acknowledgement}
The work is supported by NSF DMS-1820827. XC is also partially supported by NIH and the Alfred P.
Sloan Foundation.

\bibliography{gnn}

\begin{thebibliography}{10}

\bibitem{dcnn}
James Atwood and Don Towsley.
\newblock Diffusion-convolutional neural networks.
\newblock In {\em Advances in neural information processing systems}, pages
  1993--2001, 2016.

\bibitem{baumgardner1985icosahedral}
John~R Baumgardner and Paul~O Frederickson.
\newblock Icosahedral discretization of the two-sphere.
\newblock {\em SIAM Journal on Numerical Analysis}, 22(6):1107--1115, 1985.

\bibitem{acnn}
Davide Boscaini, Jonathan Masci, Emanuele Rodol{\`a}, and Michael Bronstein.
\newblock Learning shape correspondence with anisotropic convolutional neural
  networks.
\newblock In {\em Advances in neural information processing systems}, pages
  3189--3197, 2016.

\bibitem{bronstein2017geometric}
Michael~M Bronstein, Joan Bruna, Yann LeCun, Arthur Szlam, and Pierre
  Vandergheynst.
\newblock Geometric deep learning: going beyond euclidean data.
\newblock {\em IEEE Signal Processing Magazine}, 34(4):18--42, 2017.

\bibitem{bruna2013spectral}
Joan Bruna, Wojciech Zaremba, Arthur Szlam, and Yann LeCun.
\newblock Spectral networks and locally connected networks on graphs.
\newblock {\em arXiv preprint arXiv:1312.6203}, 2013.

\bibitem{facealignment}
Adrian Bulat and Georgios Tzimiropoulos.
\newblock How far are we from solving the 2d \& 3d face alignment problem? (and
  a dataset of 230,000 3d facial landmarks).
\newblock In {\em International Conference on Computer Vision}, 2017.

\bibitem{cao2017realtime}
Zhe Cao, Tomas Simon, Shih-En Wei, and Yaser Sheikh.
\newblock Realtime multi-person 2d pose estimation using part affinity fields.
\newblock In {\em Proceedings of the IEEE Conference on Computer Vision and
  Pattern Recognition}, pages 7291--7299, 2017.

\bibitem{face_ldmk_blessing}
Dong Chen, Xudong Cao, Fang Wen, and Jian Sun.
\newblock Blessing of dimensionality: High-dimensional feature and its
  efficient compression for face verification.
\newblock In {\em Proceedings of the IEEE conference on computer vision and
  pattern recognition}, pages 3025--3032, 2013.

\bibitem{chung1997spectral}
Fan~RK Chung and Fan~Chung Graham.
\newblock {\em Spectral graph theory}.
\newblock Number~92. American Mathematical Soc., 1997.

\bibitem{coates2011selecting}
Adam Coates and Andrew~Y Ng.
\newblock Selecting receptive fields in deep networks.
\newblock In {\em Advances in neural information processing systems}, pages
  2528--2536, 2011.

\bibitem{s2cnn}
Taco~S Cohen, Mario Geiger, Jonas K{\"o}hler, and Max Welling.
\newblock Spherical cnns.
\newblock {\em arXiv preprint arXiv:1801.10130}, 2018.

\bibitem{spherenet}
Benjamin Coors, Alexandru Paul~Condurache, and Andreas Geiger.
\newblock Spherenet: Learning spherical representations for detection and
  classification in omnidirectional images.
\newblock In {\em Proceedings of the European Conference on Computer Vision
  (ECCV)}, pages 518--533, 2018.

\bibitem{aam}
Timothy~F. Cootes, Gareth~J. Edwards, and Christopher~J. Taylor.
\newblock Active appearance models.
\newblock {\em IEEE Transactions on pattern analysis and machine intelligence},
  23(6):681--685, 2001.

\bibitem{chebnet}
Micha{\"e}l Defferrard, Xavier Bresson, and Pierre Vandergheynst.
\newblock Convolutional neural networks on graphs with fast localized spectral
  filtering.
\newblock In {\em Advances in neural information processing systems}, pages
  3844--3852, 2016.

\bibitem{fer_cnn_ding2017}
Hui Ding, Shaohua~Kevin Zhou, and Rama Chellappa.
\newblock Facenet2expnet: Regularizing a deep face recognition net for
  expression recognition.
\newblock In {\em 2017 12th IEEE International Conference on Automatic Face \&
  Gesture Recognition (FG 2017)}, pages 118--126. IEEE, 2017.

\bibitem{sphericalcnn}
Carlos Esteves, Christine Allen-Blanchette, Ameesh Makadia, and Kostas
  Daniilidis.
\newblock Learning so (3) equivariant representations with spherical cnns.
\newblock In {\em Proceedings of the European Conference on Computer Vision
  (ECCV)}, pages 52--68, 2018.

\bibitem{splinecnn}
Matthias Fey, Jan Eric~Lenssen, Frank Weichert, and Heinrich M{\"u}ller.
\newblock Splinecnn: Fast geometric deep learning with continuous b-spline
  kernels.
\newblock In {\em Proceedings of the IEEE Conference on Computer Vision and
  Pattern Recognition}, pages 869--877, 2018.

\bibitem{gama2019stability}
Fernando Gama, Joan Bruna, and Alejandro Ribeiro.
\newblock Stability properties of graph neural networks.
\newblock {\em arXiv preprint arXiv:1905.04497}, 2019.

\bibitem{mpnns}
Justin Gilmer, Samuel~S Schoenholz, Patrick~F Riley, Oriol Vinyals, and
  George~E Dahl.
\newblock Neural message passing for quantum chemistry.
\newblock In {\em Proceedings of the 34th International Conference on Machine
  Learning-Volume 70}, pages 1263--1272. JMLR. org, 2017.

\bibitem{fer13_dataset}
Ian~J Goodfellow, Dumitru Erhan, Pierre~Luc Carrier, Aaron Courville, Mehdi
  Mirza, Ben Hamner, Will Cukierski, Yichuan Tang, David Thaler, Dong-Hyun Lee,
  et~al.
\newblock Challenges in representation learning: A report on three machine
  learning contests.
\newblock In {\em International Conference on Neural Information Processing},
  pages 117--124. Springer, 2013.

\bibitem{fer_cnn_guo2016deep}
Yanan Guo, Dapeng Tao, Jun Yu, Hao Xiong, Yaotang Li, and Dacheng Tao.
\newblock Deep neural networks with relativity learning for facial expression
  recognition.
\newblock In {\em 2016 IEEE International Conference on Multimedia \& Expo
  Workshops (ICMEW)}, pages 1--6. IEEE, 2016.

\bibitem{graphsage}
Will Hamilton, Zhitao Ying, and Jure Leskovec.
\newblock Inductive representation learning on large graphs.
\newblock In {\em Advances in neural information processing systems}, pages
  1024--1034, 2017.

\bibitem{resnet}
Kaiming He, Xiangyu Zhang, Shaoqing Ren, and Jian Sun.
\newblock Deep residual learning for image recognition.
\newblock In {\em Proceedings of the IEEE conference on computer vision and
  pattern recognition}, pages 770--778, 2016.

\bibitem{densenet}
Gao Huang, Zhuang Liu, Laurens Van Der~Maaten, and Kilian~Q Weinberger.
\newblock Densely connected convolutional networks.
\newblock In {\em Proceedings of the IEEE conference on computer vision and
  pattern recognition}, pages 4700--4708, 2017.

\bibitem{isufi2020edgenets}
Elvin Isufi, Fernando Gama, and Alejandro Ribeiro.
\newblock Edgenets: Edge varying graph neural networks.
\newblock {\em arXiv preprint arXiv:2001.07620}, 2020.

\bibitem{fer_lmdk_jeong2018driver}
Mira Jeong and Byoung~Chul Ko.
\newblock Driver’s facial expression recognition in real-time for safe
  driving.
\newblock {\em Sensors}, 18(12):4270, 2018.

\bibitem{ugscnn}
Chiyu Jiang, Jingwei Huang, Karthik Kashinath, Philip Marcus, Matthias
  Niessner, et~al.
\newblock Spherical cnns on unstructured grids.
\newblock {\em arXiv preprint arXiv:1901.02039}, 2019.

\bibitem{kay2017kinetics}
Will Kay, Joao Carreira, Karen Simonyan, Brian Zhang, Chloe Hillier, Sudheendra
  Vijayanarasimhan, Fabio Viola, Tim Green, Trevor Back, Paul Natsev, et~al.
\newblock The kinetics human action video dataset.
\newblock {\em arXiv preprint arXiv:1705.06950}, 2017.

\bibitem{skar_deep_ke2017new}
Qiuhong Ke, Mohammed Bennamoun, Senjian An, Ferdous Sohel, and Farid Boussaid.
\newblock A new representation of skeleton sequences for 3d action recognition.
\newblock In {\em Proceedings of the IEEE conference on computer vision and
  pattern recognition}, pages 3288--3297, 2017.

\bibitem{keriven2019universal}
Nicolas Keriven and Gabriel Peyr{\'e}.
\newblock Universal invariant and equivariant graph neural networks.
\newblock In {\em Advances in Neural Information Processing Systems}, pages
  7090--7099, 2019.

\bibitem{skar_deep_kim2017interpretable}
Tae~Soo Kim and Austin Reiter.
\newblock Interpretable 3d human action analysis with temporal convolutional
  networks.
\newblock In {\em 2017 IEEE conference on computer vision and pattern
  recognition workshops (CVPRW)}, pages 1623--1631. IEEE, 2017.

\bibitem{gcn}
Thomas~N Kipf and Max Welling.
\newblock Semi-supervised classification with graph convolutional networks.
\newblock {\em arXiv preprint arXiv:1609.02907}, 2016.

\bibitem{cayleynet}
Ron Levie, Federico Monti, Xavier Bresson, and Michael~M Bronstein.
\newblock Cayleynets: Graph convolutional neural networks with complex rational
  spectral filters.
\newblock {\em IEEE Transactions on Signal Processing}, 67(1):97--109, 2018.

\bibitem{agcn}
Ruoyu Li, Sheng Wang, Feiyun Zhu, and Junzhou Huang.
\newblock Adaptive graph convolutional neural networks.
\newblock In {\em Thirty-second AAAI conference on artificial intelligence},
  2018.

\bibitem{liao2019lanczosnet}
Renjie Liao, Zhizhen Zhao, Raquel Urtasun, and Richard Zemel.
\newblock Lanczosnet: Multi-scale deep graph convolutional networks.
\newblock {\em ICLR}, 2019.

\bibitem{skar_deep_liu2016spatio}
Jun Liu, Amir Shahroudy, Dong Xu, and Gang Wang.
\newblock Spatio-temporal lstm with trust gates for 3d human action
  recognition.
\newblock In {\em European conference on computer vision}, pages 816--833.
  Springer, 2016.

\bibitem{genie}
Ziqi Liu, Chaochao Chen, Longfei Li, Jun Zhou, Xiaolong Li, Le~Song, and Yuan
  Qi.
\newblock Geniepath: Graph neural networks with adaptive receptive paths.
\newblock In {\em Proceedings of the AAAI Conference on Artificial
  Intelligence}, volume~33, pages 4424--4431, 2019.

\bibitem{ck+_dataset}
Patrick Lucey, Jeffrey~F Cohn, Takeo Kanade, Jason Saragih, Zara Ambadar, and
  Iain Matthews.
\newblock The extended cohn-kanade dataset (ck+): A complete dataset for action
  unit and emotion-specified expression.
\newblock In {\em 2010 ieee computer society conference on computer vision and
  pattern recognition-workshops}, pages 94--101. IEEE, 2010.

\bibitem{maron2019provably}
Haggai Maron, Heli Ben-Hamu, Hadar Serviansky, and Yaron Lipman.
\newblock Provably powerful graph networks.
\newblock In {\em Advances in Neural Information Processing Systems}, pages
  2153--2164, 2019.

\bibitem{maron2019invariant}
Haggai Maron, Heli Ben-Hamu, Nadav Shamir, and Yaron Lipman.
\newblock Invariant and equivariant graph networks.
\newblock 2019.

\bibitem{masci2015geodesic}
Jonathan Masci, Davide Boscaini, Michael Bronstein, and Pierre Vandergheynst.
\newblock Geodesic convolutional neural networks on riemannian manifolds.
\newblock In {\em Proceedings of the IEEE international conference on computer
  vision workshops}, pages 37--45, 2015.

\bibitem{meng2017identity}
Zibo Meng, Ping Liu, Jie Cai, Shizhong Han, and Yan Tong.
\newblock Identity-aware convolutional neural network for facial expression
  recognition.
\newblock In {\em 2017 12th IEEE International Conference on Automatic Face \&
  Gesture Recognition (FG 2017)}, pages 558--565. IEEE, 2017.

\bibitem{monet}
Federico Monti, Davide Boscaini, Jonathan Masci, Emanuele Rodola, Jan Svoboda,
  and Michael~M Bronstein.
\newblock Geometric deep learning on graphs and manifolds using mixture model
  cnns.
\newblock In {\em Proceedings of the IEEE Conference on Computer Vision and
  Pattern Recognition}, pages 5115--5124, 2017.

\bibitem{fer_lmdk_morales2019use}
E~Morales-Vargas, CA~Reyes-Garc{\'\i}a, and Hayde Peregrina-Barreto.
\newblock On the use of action units and fuzzy explanatory models for facial
  expression recognition.
\newblock {\em PloS one}, 14(10), 2019.

\bibitem{morris2019weisfeiler}
Christopher Morris, Martin Ritzert, Matthias Fey, William~L Hamilton, Jan~Eric
  Lenssen, Gaurav Rattan, and Martin Grohe.
\newblock Weisfeiler and leman go neural: Higher-order graph neural networks.
\newblock In {\em Proceedings of the AAAI Conference on Artificial
  Intelligence}, volume~33, pages 4602--4609, 2019.

\bibitem{modelnet_pointnet}
Charles~R Qi, Hao Su, Kaichun Mo, and Leonidas~J Guibas.
\newblock Pointnet: Deep learning on point sets for 3d classification and
  segmentation.
\newblock In {\em Proceedings of the IEEE conference on computer vision and
  pattern recognition}, pages 652--660, 2017.

\bibitem{modelnet_3dconv}
Charles~R Qi, Hao Su, Matthias Nie{\ss}ner, Angela Dai, Mengyuan Yan, and
  Leonidas~J Guibas.
\newblock Volumetric and multi-view cnns for object classification on 3d data.
\newblock In {\em Proceedings of the IEEE conference on computer vision and
  pattern recognition}, pages 5648--5656, 2016.

\bibitem{modelnet_pointnet++}
Charles~Ruizhongtai Qi, Li~Yi, Hao Su, and Leonidas~J Guibas.
\newblock Pointnet++: Deep hierarchical feature learning on point sets in a
  metric space.
\newblock In {\em Advances in neural information processing systems}, pages
  5099--5108, 2017.

\bibitem{qiu2018dcfnet}
Q~Qiu, X~Cheng, R~Calderbank, and G~Sapiro.
\newblock Dcfnet: Deep neural network with decomposed convolutional filters.
\newblock In {\em International Conference Machine Learning}, 2018.

\bibitem{survey_sbar}
Bin Ren, Mengyuan Liu, Runwei Ding, and Hong Liu.
\newblock A survey on 3d skeleton-based action recognition using learning
  method.
\newblock {\em arXiv preprint arXiv:2002.05907}, 2020.

\bibitem{nn4g}
Franco Scarselli, Marco Gori, Ah~Chung Tsoi, Markus Hagenbuchner, and Gabriele
  Monfardini.
\newblock The graph neural network model.
\newblock {\em IEEE Transactions on Neural Networks}, 20(1):61--80, 2008.

\bibitem{schonsheck2018parallel}
Stefan~C Schonsheck, Bin Dong, and Rongjie Lai.
\newblock Parallel transport convolution: A new tool for convolutional neural
  networks on manifolds.
\newblock {\em arXiv preprint arXiv:1805.07857}, 2018.

\bibitem{shahroudy2016ntu}
Amir Shahroudy, Jun Liu, Tian-Tsong Ng, and Gang Wang.
\newblock Ntu rgb+ d: A large scale dataset for 3d human activity analysis.
\newblock In {\em Proceedings of the IEEE conference on computer vision and
  pattern recognition}, pages 1010--1019, 2016.

\bibitem{gat}
Petar Veli{\v{c}}kovi{\'c}, Guillem Cucurull, Arantxa Casanova, Adriana Romero,
  Pietro Lio, and Yoshua Bengio.
\newblock Graph attention networks.
\newblock {\em arXiv preprint arXiv:1710.10903}, 2017.

\bibitem{body_ldmk_liegroup}
Raviteja Vemulapalli, Felipe Arrate, and Rama Chellappa.
\newblock Human action recognition by representing 3d skeletons as points in a
  lie group.
\newblock In {\em Proceedings of the IEEE conference on computer vision and
  pattern recognition}, pages 588--595, 2014.

\bibitem{skar_hand_chellappa}
Raviteja Vemulapalli, Felipe Arrate, and Rama Chellappa.
\newblock Human action recognition by representing 3d skeletons as points in a
  lie group.
\newblock In {\em Proceedings of the IEEE conference on computer vision and
  pattern recognition}, pages 588--595, 2014.

\bibitem{skar_hand_wang2012mining}
Jiang Wang, Zicheng Liu, Ying Wu, and Junsong Yuan.
\newblock Mining actionlet ensemble for action recognition with depth cameras.
\newblock In {\em 2012 IEEE Conference on Computer Vision and Pattern
  Recognition}, pages 1290--1297. IEEE, 2012.

\bibitem{survey_gnn}
Zonghan Wu, Shirui Pan, Fengwen Chen, Guodong Long, Chengqi Zhang, and S~Yu
  Philip.
\newblock A comprehensive survey on graph neural networks.
\newblock {\em IEEE Transactions on Neural Networks and Learning Systems},
  2020.

\bibitem{xu2019howpowerful}
Keyulu Xu, Weihua Hu, Jure Leskovec, and Stefanie Jegelka.
\newblock How powerful are graph neural networks?
\newblock {\em ICLR}, 2019.

\bibitem{stgcn}
Sijie Yan, Yuanjun Xiong, and Dahua Lin.
\newblock Spatial temporal graph convolutional networks for skeleton-based
  action recognition.
\newblock In {\em Thirty-second AAAI conference on artificial intelligence},
  2018.

\bibitem{gaan}
Jiani Zhang, Xingjian Shi, Junyuan Xie, Hao Ma, Irwin King, and Dit-Yan Yeung.
\newblock Gaan: Gated attention networks for learning on large and
  spatiotemporal graphs.
\newblock {\em arXiv preprint arXiv:1803.07294}, 2018.

\end{thebibliography}
\bibliographystyle{plain}

\appendix

\setcounter{table}{0}
\setcounter{figure}{0}
\setcounter{lemma}{0}
\renewcommand{\thetable}{A.\arabic{table}}
\renewcommand{\thefigure}{A.\arabic{figure}}
\renewcommand{\thelemma}{A.\arabic{lemma}}


\section*{Appendix}

\section{Proofs}\label{sec:proofs}

\subsection{Details and proofs in Sec. \ref{subsec:unified}}

\subsubsection{ Locally connected GNN}

Specifically, the construction in \cite{coates2011selecting,bruna2013spectral} assumes that
$u$ and $u'$ belongs to the graph of different scales, $u'$ is on the fine graph, and $u$ is on a coarse-grained layer produced by clustering of indices of the graph of the input layer. 
If one generalize the construction to allow over-lapping of the receptive fields, and assume no pooling or coarse-graining of the graph, then the non-zero parameters are of the number 
\[
\sum_{u \in V} |N_{u}| \cdot C C' = n p \cdot C C',
\]
where $n = |V|$, $p$ is the average patch size $|N_u|$, and $C$ and $C'$ are the number of input and output feature channels.

\subsubsection{  ChebNet/GCN, GAT and Edgenet}

$\bullet$ Chebet/GCN

In view of \eqref{eq:gnn},
ChebNet \cite{chebnet} makes use of the graph adjacency matrix to construct $M$. 
Specifically, 
$A_{sym}:= D^{-1/2} A D^{-1/2}$ is the symmetrized graph adjacency matrix (possibly including self-edge, then $A$ equals original $A$ plus $I$),
and $L_{sym} := I - A_{sym}$  has spectral decomposition $L_{sym} = \Psi \Lambda \Psi^T$.
Let $\tilde{L} = \alpha_1 I + \alpha_2 L_{sym}$ be the rescaled and re-centered graph Laplacian such that the eigenvalues are between $[-1,1]$,
$\alpha_1, \alpha_2$ fixed constants.
Then, written in $n$-by-$n$ matrix form,
\begin{equation}\label{eq:M-cheb-1}
M_{c',c} = \sum_{l=0}^{L-1}  \theta_l(c',c) T_l( \tilde{L} ),
\quad \theta_l(c',c) \in \R,
\end{equation}
where $T_l(\cdot)$ is Chebshev polynomial of degree $l$.
As $A_{sym}$ and then $\tilde{L}$ are given by the graph, only $\theta_l$'s are trainable, 
thus the number of parameters are 
\[
L \cdot CC'.
\]

GCN \cite{gcn} is a special case of ChebNet.
Take $L=2$ in \eqref{eq:M-cheb-1}, and tie the choice of $\theta_0$ and $\theta_1$,
\[
M_{c',c} =  \theta(c',c) (\alpha_1' I + \alpha_2' A_{sym}) =:  \theta(c',c) \tilde{A}, 
\quad \alpha_1', \alpha_2' \text{ fixed constants,}
\]
where $\theta(c',c)$ is trainable. 
This factorized form leads to the linear part of the layer-wise mapping as $Y = \tilde{A} X \Theta$ written in matrix form, 
where $ \tilde{A}$ is $n$-by-$n$ matrix defined as above,
$X$ ($Y$) is $n$-by-$C'$ (-$C$) array, $\Theta$ is $C'$-by-$C$ matrix.
 The model complexity is $C C'$ which are the parameters in $\Theta$.

 $\bullet$ GAT
 
  In GAT \cite{gat}, $R$ being the number of attention heads, the graph convolution operator in one GNN layer can be written as (omitting bias and non-linear mapping)
\begin{equation}\label{eq:def-gat}
Y = \sum_{r=1}^R {\cal A}^{(r)} X \Theta_r, 
\quad {\cal A}^{(r)}_{u,v} = \frac{ e^{c^{(r)}_{uv} } }{ \sum_{v' \in N_u^{(1)}}  e^{ c^{(r)}_{uv'} } }, 
\quad c_{uv}^{(r)} = \sigma( (a^{(r)})^T [ W^{(r)}X_u,  W^{(r)}X_v ]  ),
\end{equation}
where $\{ W^{(r)}, a^{(r)}\}$ are the trainable parametrization of attention graph affinity mechanism ${\cal A}^{(r)}$, which constructs non-negative affinities between graph nodes $u$ and $v$ adaptively from the input graph node feature $X$. 
In particular, ${\cal A}^{(r)}$ shares sparsity pattern as the graph topology,
that is, ${\cal A}^{(r)}(u,u') \neq 0$ only when $u' \in N_u^{(1)}$.

In the original GAT, $\Theta_r = W^{(r)} {\bf C}^{(r)}$, where ${\bf C}^{(r)}$'s are fixed matrices such that the output from $r$-th head is concatenated into the output $Y$ across $r=1,\cdots, R$. 
Variants of GAT adopt channel mixing across heads,
e.g. a generalization of GAT in \cite{isufi2020edgenets} uses extra trainable $\Theta_r$ in \eqref{eq:def-gat} independent from $W^{(k)}$.
\cite{isufi2020edgenets} also proposed higher-order GAT by considering powers of the affinity matrix ${\cal A}^{(r)} $ as well as the edge-varying version
(c.f. Eqn. (36)(39) in \cite{isufi2020edgenets}). 
 As this higher-order GAT and the edge-varying counterpart are special cases of the edgy-varying GNN, we cover this case in 
 Proposition \ref{prop:larger-than-cheb} 3).

 The model complexity of GAT:
 In the original GAT where $\Theta_r$ is tied with $W^{(r)}$,
 the number of parameters in one layer is 
 $R(C_0 C' + 2C_0)$, where $R$ is the number of attention heads, $C = C_0 R$, and $W^{(r)}: \R^{C'} \to \R^{C_0}$. 
 When $\Theta_r$ are free from $\{ W^{(r)}, a^{(r)}\}$  in \eqref{eq:def-gat}, the number of parameters is $R( CC'+C_0C' + 2C_0) \le R(2CC'+2C)$,
 where $W^{(r)}$ maps to dimension $C_0$ and $\Theta_r$ maps to dimension $C$.

$\bullet$ EdgeNet (Edge-varying GCN)

Per Eqn. (1)(8) in \cite{isufi2020edgenets}, the edge-varying GNN layer mapping can be written as 
\begin{equation}\label{eq:def-edge-varying}
Y = \sum_{r=0}^{L-1} \left( \prod_{k=0}^r \Phi_k \right) X \Theta_r,
\end{equation}
where $\Phi_0$ is an $n$-by-$n$ diagonal matrix, and $\Phi_k$, $k=1,\cdots, r$,
are supported on $N_u^{(1)}$ of each node $u$. 
The trainable parameters are $\{\Phi_k\}_{k=0}^R$ and $\{\Theta_r\}_{r=0}^R$, $\Theta_r : \R^{C'}\to \R^C$. 
Edge-varying GAT implements polynomials of averaging filters,
and general edge-varying GNN takes product of arbitrary 1-order filters.
The proof shows that EdgeNet layer is a special case of L3Net layer,
while restricting $B_k$ to be of the product form \eqref{eq:Bk-prod}
rather than freely supported on $N_u^{(d_k)}$ for user-specified order $(d_1,\cdots,d_K)$
is a non-trivial restriction.

The trainable parameters: $\Theta_r$ has $LCC'$ many, 
$\Phi_0$ has $n$,
and $\Phi_k$, $k=1,\cdots, L-1$ each has  $n p^{(1)}$ many, $p^{(1)}$ being the average size o 1-neighborhood of nodes.
Thus the total number of parameters is 
\[
LCC' + n + (L-1) n p^{(1)} 
\sim  L ( CC' + n p^{(1)} ).
\]

\begin{proof}[Proof of Proposition \ref{prop:larger-than-cheb}]
Part (1): Since GCN is a special case of ChebNet, 
it suffices to prove that \eqref{eq:M-cheb-1} can be expressed in the form of L3Net \eqref{eq:M-ours} for some $K$.
By definition of $\tilde{L}$, mathematically equivalently, 
\begin{equation}\label{eq:M-cheb-2}
M_{c',c} = \sum_{l=0}^{L-1}  \theta_l(c',c) T_l(  \alpha_1 I + \alpha_2 L )
= \sum_{l=0}^{L-1}  \theta_l(c',c) T_l(  \alpha_1 I + \alpha_2 (I -A_{sym}) )
= \sum_{l=0}^{L-1} \beta_l(c',c) A_{sym}^l,
\end{equation}
where the coefficients $\beta_l$'s are determined by $\theta_l$'s,
per $(c',c)$.
Since $A_{sym}^l$ propagates to the $l$-th order neighborhood of any node,
setting $B_{k}(u',u) = A_{sym}^{k-1}(u',u)$, $B_k(u',u)$ is non-zero when $u' \in N_u^{(k-1)}$,
$1 \le k \le K := L$,
and then setting $a_k(c',c) = \beta_{k-1}(c',c)$ 
gives  \eqref{eq:M-cheb-1} in the form of 
\eqref{eq:M-ours}.

Part (2): We consider \eqref{eq:def-gat} as the GAT model.
Recall that $\Theta_r: \R^{C'} \to \R^C$, then \eqref{eq:def-gat} can be re-written in the form of \eqref{eq:gnn} by letting
\[
M(u', u; c', c) 
= \sum_{r=1}^R {\cal A}^{(r)}(u',u) \Theta_r(c',c),
\]
which is a special case of \eqref{eq:M-ours} where
 $R=K$, ${\cal A}^{(k)} = B_k$ and $\Theta_k = a_k$.
Since ${\cal A}^{(r)}(u,u')$ as a function of $u'$ is supported on $u' \in N_u^{(1)}$,
\eqref{eq:def-gat} belongs to the L3Net model \eqref{eq:M-ours} where $d_1 = \cdots = d_K = 1$,  
 in addition to that $B_k$ must be of the attention affinity form, 
 i.e. built from the attention coefficients $c_{uv}^{(r)}$ computed from input $X$ via parameters $\{ W^{(r)}, a^{(r)}\}$.

Part (3):
Comparing with \eqref{eq:gnn}\eqref{eq:M-ours},
we have that \eqref{eq:def-edge-varying}
is a special case of L3Net \eqref{eq:M-ours} 
by letting $K = L$,
\begin{equation}\label{eq:Bk-prod}
B_k = \prod_{k'=0}^{k-1} \Phi_{k'},
\end{equation}
$a_{k}=\Theta_{k-1}$, 
and $d_k=k-1$ for $k=1,\cdots,K$.
\end{proof}

\subsubsection{ Standard and geometrical CNN's}

Standard CNN on $\R^d$, e.g. $d=1$ for audio signal and $d=2$ for image data,
applies a discretized convolution to the input data in each convolutional layer,
which can be written as (omitting bias which is added per $c$, and the non-linear activation)
\begin{equation}\label{eq:standard-cnn}
y(u,c) = \sum_{c' \in [C']} \sum_{u' \in U} w_{c',c}(u'-u)  x(u', c'),
\end{equation}
where $U$ is a grid on $\R^d$.
We write in the way of ``anti-convolution'', which has ``$u'-u$'' rather than ``$u-u'$'',
but the definition is equivalent. 
For audio and image data, $U$ is usually a regular mesh with evenly sampled grid points,
and proper boundary conditions are applied when computing $y(u,c)$ at a boundary grid point $u$.
E.g., boundary can be handled by standard padding as in CNN.
As the convolutional filters $w_{c',c}$ are compactly supported, the summation of $u'$ is on a neighborhood of $u$. 

More generally, CNN's on non-Euclidean domains are constructed when spatial points are sampled on an irregular mesh in $\R^d$, 
e.g., a 2D surface in $\R^3$. 
The generalization of \eqref{eq:standard-cnn} is by defining the ``patch operator" \cite{masci2015geodesic}
 which pushes a template filter $w$ on a regular mesh on $\R^d$, $d$ being the intrinsic dimensionality of the sampling domain,
to the irregular mesh in the ambient space that have coordinates on local charts. 
Specifically, for a mesh of 2D surface in 3D, 
$d=2$, and $w$ is a template convolutional filter on $\R^2$.
For any local cluster of 3D mesh points $N_u$  around a point $u$, 
the patch operator ${\cal P}_u$ provides $({\cal P}_u w)(u')$ for $u' \in N_u$
by certain interpolation scheme on the local chart.
The operator ${\cal P}_u$ is linear in $w$, and possibly trainable.
As a result, in mesh-based geometrical CNN,
\begin{equation}\label{eq:geo-cnn}
y(u,c) = \sum_{c' \in [C']} \sum_{u'} ({\cal P}_u w_{c',c})(u')  x(u', c'),
\end{equation}
and one can see that in Euclidean space taking $({\cal P}_u w)(u') = w(u'-u)$
reduces \eqref{eq:geo-cnn} to the standard CNN as in \eqref{eq:standard-cnn}.

In both \eqref{eq:standard-cnn} and \eqref{eq:geo-cnn},
spatial low-rank decomposition of the filters $w_{c',c}$ can be imposed \cite{qiu2018dcfnet}.
This introduces a set of bases $\{ b_k \}_k$ over space that linearly span the filters $w_{c',c}$.
For standard CNN in $\R^d$, $b_k$ are basis filters on $\R^d$,
and for geometrical CNN, they are defined on the reference domain in $\R^d$ same as $w_{c',c}$, where $d$ is the intrinsic dimension.
Suppose $w_{c',c} = \sum_{k=1}^K \beta_{k,(c',c)} b_k$
for coefficients $\beta_{k,(c',c)}$,
by linearity,
\eqref{eq:geo-cnn} becomes
\begin{align}\label{eq:geo-dcf-cnn}
y(u,c)  = \sum_{c' \in [C']} \sum_{u'} \sum_{k=1}^K \beta_{k,(c',c)}  ({\cal P}_u b_k )(u')  x(u', c'),
\end{align}
and similarly for \eqref{eq:standard-cnn}.
The trainable parameters in \eqref{eq:geo-dcf-cnn}
are $\beta_{k,(c',c)} $ and the basis filters $b_k$'s,
the former has $K CC'$ parameters, 
and the latter has $\sum_{k} p_k$,
where $p_k$ is the size of the support of $b_k$ in $\R^d$.
Suppose the average size is $p$, then the number of parameters is $Kp$. 
This gives the total number of parameters as
\[
K C C' + Kp.
\]

\begin{proof}[Proof of Proposition \ref{prop:larger-than-dcfcnn}]
Since standard CNN is a special case of geometrical CNN \ref{eq:geo-cnn},
we only consider the latter.
Assuming low-rank filter decomposition,
the convolutional mapping is \eqref{eq:geo-dcf-cnn}.
Comparing to the GNN layer mapping defined in \eqref{eq:gnn}, one sees that
\[
M(u',u; c',c) = \sum_{k=1}^K \beta_{k,(c',c)}  ({\cal P}_u b_k )(u'),
\]
which equals \eqref{eq:M-ours} if setting $B_k(u',u) = ({\cal P}_u b_k )(u')$
and $a_k(c',c) = \beta_{k,(c',c)} $. 
\end{proof}

\subsubsection{Strong regularization limit}

\begin{proof}[Proof of Proposition \ref{prop:large-reg}]
The constrained minimization of ${\cal R}$ 
defined in  \eqref{eq:def-reg} separates for each $u, k$,
and the minimization of $b_u^{(k)}$ is given by
\begin{equation}\label{eq:min-w}
\min_{w : N_u^{(d_k)} \to \R} w^T L_u^{(k)} w,
\quad \text{ s.t. } \|w\|_2 \ge \alpha_{u,k} > 0.
\end{equation}
For each $u, k$,
the local Dirichlet graph Laplacian  $L_u^{(k)}$  has eigen-decomposition $L_u^{(k)} = \Psi_u^{(k)} \Lambda_u^{(k)} (\Psi_u^{(k)})^T$,
where  $(\Psi_u^{(k)})^T \Psi_u^{(k)} = I$,
 and the diagonal entries of $\Lambda_u^{(k)}$ are eigenvalues of $L_u^{(k)}$,
 which are all $ \ge 0$ and sorted in increasing order.
By the variational property of eigenvalues,
 the minimizer of $w$ in \eqref{eq:min-w} is  achieved when 
 $w = \Psi_u^{(k)} ( \cdot,1)$, 
 i.e., the eigenvector associated with the smallest eigenvalue of $L_u^{(k)}$.  
By that the local subgraph is connected, 
this smallest eigenvalue has single multiplicity,
and the eigenvector is the Perron-Frobenius vector which does not change sign. 
The claim holds for arbitrary $\alpha_{u,k} > 0$ since eigenvector is 
defined up to a constant multiplication. 
\end{proof}

\subsection{Proofs in Sec. \ref{subsec:expressive}}

\begin{figure}
    \centering
    \includegraphics[trim={113pt 530pt 195pt 10pt},clip,width=.7\linewidth]{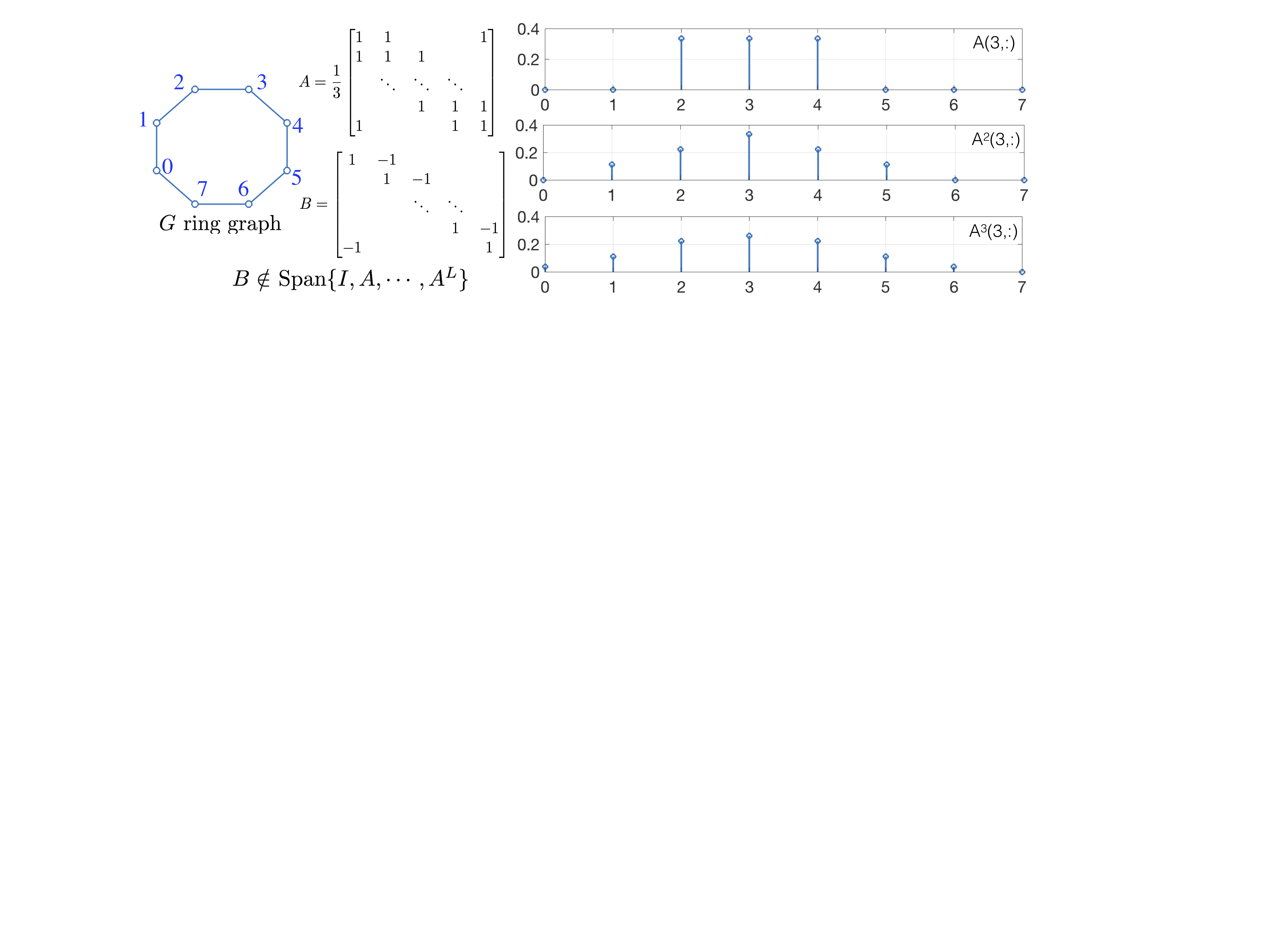}
    \caption{ A ring graph with $8$ nodes.
        Polynomials of graph adjacency matrix $A$
        (or Laplacian matrix)
        preserve symmetry of mirroring around any node, e.g., node 3, and can cannot express a local filter $B$}
    \label{fig:ring_expressive}
\end{figure}

\begin{proof}[Proof of Proposition \ref{prop:expressive}]
Part 1): Let the graph be the ring graph with $n$ nodes,
and each node has 2 neighbors,
$n$=8 as shown in Fig. \ref{fig:diag1} (right).
We index the nodes as $u = 0, \dots, n-1$ and allows addition/subtraction of $u-v$ (mod $n$).
Let $B$ be the ``difference'' filter $B(u',u) = 1$ when $u' = u$ and $-1$ when $u' = u+1$.
We show that $B \neq f(A)$ for any $f$,
and in contrast, setting this $B$ as the basis in \eqref{eq:M-ours} expresses the filter with $K=1$.

To prove that $B \neq f(A)$ for any $f$,
let $\pi_u$ be the permutation of the $n$ nodes such that 
$\pi_u( u+v) = ( u-v) $ for all $v$, 
i.e., mirror flip the ring around the node $u$.
By construction, the graph topology of the ring graph is preserved under $\pi_u$,
that is, $A_{\pi_u}:=\pi_u A \pi_u^T = A$,
whether $A$ is the 0/1 value adjacency matrix or the symmetrically normalized one 
$A_{sym}=D^{-1/2}AD^{-1/2}$ ($D$ is constant on diagonal)
or other normalized version as long as the relation $A_{\pi_u}= A$ holds.
By Lemma \ref{lemma:node-permutation} 1), for any $f: \R \to \R$,
\[
f(A) \pi_u = f(A_{\pi_u}) \pi_u = \pi_u f(A),
\]
this means that if $B = f(A)$ for some $f$,
then $B \pi_u = \pi_u B$, which contradicts with the construction of $B$.

Part 2): Consider the two distributions of graph signals on the ring graph in 1),
which we call ``upwind/downwind'' signals: 
$X_{up}$ consists of finite superpositions of functions on the ring graph 
which are periodic, smoothly increasing from 0 to 1 and then dropping to zero.
Signals in $X_{up}$ are under certain distribution,
and $X_{down}$ consists of the signals that can be produced by mirror-flipping the upwind signals.
That is, denoting $x_{up}$ ($x_{down}$) an upwind (downwind) signal, 
$\pi_u$ the permutation as in 1) around any node $u$, 
then
\[
 \pi_u x_{up}  \overset{\text{dist.}}{=}  x_{down},
\]
where $ \overset{\text{dist.}}{=} $ means equaling in distribution. 
Example signals of the two classes as illustrated in Fig. \ref{fig:updown}.

Same as in 1), 
by construction $A_{\pi_u} = A$.
Let $F^{(L)}$ be the mapping to the $L$-th layer spectral GNN feature,
for $x_{up}$ an upwind signal,
Lemma \ref{lemma:node-permutation} 2) gives that
\[
F^{(L)}[A] \pi_u x_{up} 
= F^{(L)}[A_{\pi_u}] \pi_u x_{up} 
= \pi_u F^{(L)}[A] x_{up}.
\]
The last layer applies group invariant operator $U$, then
\[
U F^{(L)}[A] \pi_u x_{up}  
= U \pi_u F^{(L)}[A] x_{up}
= U F^{(L)}[A] x_{up},
\]
this gives that 
\[
U F^{(L)}[A] x_{down}  
 \overset{\text{dist.}}{=}
 U F^{(L)}[A] \pi_u x_{up}  
 = U F^{(L)}[A] x_{up},
\]
which means that the final output deep feature via $U F^{(L)}[A]$
are statistically the same for the input signals from the two classes. '

Meanwhile,
the difference local filter $B$ in the proof of 1) 
can extract feature to differentiate the two classes,
and then L3Net with 1 layer and 1 basis suffices to distinguish the $X_{up}$
and $X_{down}$ signals. 
\end{proof}
\begin{lemma}[Permutation equivariance, Proposition 1 in \cite{gama2019stability}]
\label{lemma:node-permutation}
Let  $A$ be the (possibly normalized) graph adjacency matrix,
for any input signal $x: V \to \R$,
and  $\pi \in {\cal S}_n$ a permutation of graph nodes, 

1) The spectral graph convolution mapping $f(A)$ satisfies that 
\[
f(A_\pi ) \pi  = \pi f(A),
\quad 
A_\pi := \pi A \pi^T.
\]

2) Let $F^{(l)}[A]$ be the mapping to the $l$-th layer spectral GNN feature with graph adjacency $A$, then
\[
F^{(l)}[A_{\pi}] \pi x = \pi F^{(l)}[A] x.
\]
\end{lemma}
\begin{proof}[Proof of Lemma \ref{lemma:node-permutation}]
Proved in \cite{gama2019stability}
and we reproduce with our notation for completeness.

Part 1):  Denote the $n$-by-$n$ permutation matrix also by $\pi$,
then by definition, $f(A) = U f(\Lambda) U^T$ where $A = U \Lambda U^T$ is the diagonalization and $U$ is orthogonal matrix,
thus
\[
f(A_\pi) = f( \pi U \Lambda U^T \pi^T )
=  \pi U f( \Lambda) U^T \pi^T  
= \pi f(A) \pi^T,
\] 
and this proves 1).

Part 2): Each spectral GNN layer mapping adds the bias and the node-wise non-linear activation mapping
 to the graph convolution linear operator,
which preserves the permutation equivariance. 
Recursively applying to $L$ layers proves 2).
\end{proof}

\subsection{Proofs in Sec. \ref{subsec:analysis-stable}}

\begin{proof}[Proof of Theorem \ref{thm:stable-1}]
By definition,
\[
Y(u) = \sigma(  \sum_{k=1}^K a_k \langle B_k(\cdot, u), X(\cdot) \rangle_{N_u^{(d_k)}}   + \text{bias}),
\]
then since $\sigma$ is non-expansive, $\forall u \in V$,
\begin{equation}\label{eq:bound-dY-1}
| \Delta Y(u)  | 
\le 
| \sum_{k=1}^K a_k \langle B_k(\cdot, u), \Delta X(\cdot) \rangle_{N_u^{(d_k)}} |
\le \|a\|_2  
\left( \sum_{k=1}^K  |\langle B_k(\cdot, u), \Delta X(\cdot) \rangle_{N_u^{(d_k)}} |^2  \right)^{1/2}.
\end{equation}
By that 
\begin{equation}\label{eq:CS-1}
| \langle B_k(\cdot, u), \Delta X(\cdot) \rangle_{N_u^{(d_k)}}| 
\le \|  B_k(\cdot, u) \|_{2, N_u^{(d_k)} }  \cdot \|  \Delta  X (\cdot) \|_{2, N_u^{(d_k)} },
\end{equation}
we have that 
\begin{align}
\sum_{u \in V}| \Delta Y(u)  |^2 
& \le 
\|a\|_2^2
 \sum_{u} \sum_{k=1}^K  |\langle B_k(\cdot, u), \Delta X(\cdot) \rangle_{N_u^{(d_k)}} |^2  
\nonumber \\
 & \le  \|a\|_2^2
 \sum_{u} \sum_{k=1}^K   \|  B_k(\cdot, u) \|_{2, N_u^{(d_k)} }^2  \cdot  \| \Delta  X (\cdot) \|_{2, N_u^{(d_k)} }^2
\nonumber \\
 & \le  (\|a\|_2 \beta^{(1)})^2 \sum_{u, k } \| \Delta  X (\cdot) \|_{2, N_u^{(d_k)} }^2,
 \label{eq:bound-dy2-2}
\end{align}
and observe that
\begin{align*}
& \sum_{u, k } \| \Delta  X (\cdot) \|_{2, N_u^{(d_k)} }^2
 = \sum_{k=1}^K \sum_{u \in V} \sum_{v \in N_u^{(d_k)} } |\Delta X(v)|^2
=  \sum_{k=1}^K \sum_{u,v \in V} {\bf 1}_{ \{ v \in N_u^{(d_k)} \} } |\Delta X(v)|^2 \\
& =  \sum_{k=1}^K \sum_{u,v \in V} {\bf 1}_{ \{ u \in N_v^{(d_k)} \} } |\Delta X(v)|^2
= \sum_{k=1}^K \sum_{v \in V}  | N_v^{(d_k)} | \cdot |\Delta X(v)|^2
\le Kp \sum_{v \in V}  |\Delta X(v)|^2,
\end{align*}
where we used the assumption on $Kp$ to obtain the last $\le$. 
Then  \eqref{eq:bound-dy2-2} continues as 
\[
\le (\|a\|_2 \beta^{(1)})^2 Kp  \| \Delta X\|_{2,V}^2,
\]
which proves that $\| \Delta Y\|_{2,V} \le (\|a\|_2 \beta^{(1)}) \sqrt{Kp}  \| \Delta X\|_{2,V}$ as claimed.
\end{proof}

\begin{proof}[Proof of Theorem \ref{thm:stable-2}]
Same as in the proof of Theorem \ref{thm:stable-1}, we have \eqref{eq:bound-dY-1}. 
The eigen-decomposition $L_u^{(k)} = \Psi_u^{(k)} \Lambda_u^{(k)} (\Psi_u^{(k)})^T$
has that $(\Psi_u^{(k)})^T \Psi_u^{(k)} = I$,
 and, under the connectivity condition of the subgraph, the diagonal entries of $\Lambda_u^{(k)}$ all $ > 0$.
Thus 
\[
\langle u, v \rangle_{N_u^{(d_k)}} = \langle  (\Lambda_u^{(k)})^{1/2} \Psi_u^{(k)}u,  (\Lambda_u^{(k)})^{-1/2} \Psi_u^{(k)}v \rangle_{N_u^{(d_k)}},
\]
which gives the Cauchy-Schwarz with  weighted 2-norm as 
\begin{equation}\label{eq:CS-2}
| \langle B_k(\cdot, u), \Delta X(\cdot) \rangle_{N_u^{(d_k)}}| 
\le \|  B_k(\cdot, u) \|_{L_u^{(k)} }  \cdot \|  \Delta  X (\cdot) \|_{ (L_u^{(k)} )^{-1} }.
\end{equation}
Then similarly as in \eqref{eq:bound-dy2-2}, using the definition of $\beta^{(2)}$ and the the condition with $\rho$,
we obtain that 
\begin{align}
\sum_{u \in V}| \Delta Y(u)  |^2
  \le  (\|a\|_2 \beta^{(2)})^2 \sum_{u, k } \rho^2 \| \Delta  X (\cdot) \|_{2, N_u^{(d_k)} }^2,
 \label{eq:bound-dy2-3}
\end{align}
and the rest of the proof is the same, which gives that 
\[
\sum_{u \in V}| \Delta Y(u)  |^2 \le 
(\|a\|_2 \beta^{(2)})^2 \rho^2 Kp  \| \Delta X\|_{2,V}^2,
\]
which proves the claim. 
\end{proof}

\section{Up/down-wind Classification Experiment}\label{app:updown-detail}

\subsection{Dataset Setup}
We generate the Up/Down wind dataset on both ring graph and chain graph with 64 nodes. Every node is assigned to a probability drawn from $(0, 1)$ uniform distribution. Node with probability less than $threshold=0.1$ will be assigned with a gaussian distribution with $std=1.5$. Each gaussian distribution added is masked half side. Distribution masked left half is the 'Down Wind' class, distribution masked right half is the 'Up Wind' class, as shown in left plot in Fig. \ref{fig:updown}. We then 
sum up all half distributions from different locations in each sample. We generate 5000 training samples and 5000 testing samples.

\subsection{Model architecture and training details}

{\bf  Network architectures.}  

$\bullet$ 2-gcn-layer model: 

~~~ GraphConv(1,32)-ReLU-MaxPool1d(2)-GraphConv(32,64)-ReLU-AvgPool(32)-FC(2),

$\bullet$ 1-gcn-layer model:

~~~ GraphConv(1,32)-ReLU-AvgPool(64)-FC(2),

where GraphConv can be ChebNet or L3Net.

{\bf  Traning details.}

We choose the Adam Optimizer, batch size of 100, 
set initial learning rate of $1\times10^{-3}$, make it decay by 0.1 at 80 epoch and train for 100 epoches.

\subsection{Additional results}
We report additional results using 1-gcn layer architecture in Tab. \ref{tab:up-down-gcn-1}. 
Our L3Net again shows stronger classification performance than ChebNet.

\begin{table}[htb]
    \scriptsize
    \centering
    \caption{results of 1-gcn layer models}
    \begin{tabular}{c|c|c|c||c}
    \hline
    Gnn model & order & \#params & ring graph Acc & chain graph Acc\\
    \hline
    \multirow{4}*{ChebNet }  & L=3 & 0.2k & $50.80\pm0.24$ & $50.66\pm0.21$ \\
    ~                        & L=5 & 0.3k & $51.14\pm0.21$ & $51.07\pm0.35$ \\
    ~                        & L=9 & 0.4k & $51.68\pm0.38$ & $50.96\pm0.29$ \\
    ~                        & L=30 & 1.1k & $51.37\pm0.14$ & $50.70\pm0.16$ \\
    \hline
    \multirow{2}{*}{L3Net} & 1 & 0.3k & $99.96\pm0.08$ & $99.67\pm0.12$ \\
    ~                          & 0;1;2 & 0.8k & $\mathbf{99.96\pm0.01}$ & $\mathbf{99.92\pm0.01}$ \\
    \hline
    \end{tabular}
    \label{tab:up-down-gcn-1}
\end{table}

\section{Experimental Details}

\subsection{Classification of sphere mesh data}\label{app:detail-1}

{\bf  Spherical mesh}
We conduct this experiment on icosahedral spherical mesh \cite{baumgardner1985icosahedral}. Like S2CNN \cite{s2cnn}, we project digit image onto surface of unit sphere, and follow \cite{ugscnn} by moving projected digit to equator, avoiding coordinate singularity at poles.

Here, we details the subdivision scheme of the icosahedral spherical mesh we used. Start with an unit icosahedron, this sphere discretization progressively subdivide each face into four equal triangles, which makes this discretization uniform and accurate. Plus, this scheme provides a natural downsampling strategy for networks, as it denotes the path for aggregating information from higher-level neighbor nodes to lower-level center node. We adopt the following naming convention for different mesh resolution: start with level-0($L0$) mesh(i.e., unit icosahedron), each level above is associated with a subdivision. For level-$i(L_i)$, properties of sperical mesh are:
\begin{equation}
    N_e=30\cdot4*i, N_f=20\cdot4*i, N_v=N_e-N_f+2
\end{equation}
in which $N_f,N_e,N_v$ denote number of edges, faces, and vertices.

To give a direct illustration of how many nodes each level of mesh has, we list them below,
\begin{itemize}
    \item $L0$ 12 nodes
    \item $L1$ 42 nodes
    \item $L2$ 162 nodes
    \item $L3$ 642 nodes
    \item $L4$ 2562 nodes
    \item $L5$ 10242 nodes
\end{itemize}

{\bf  Network architectures}
We use a three-stage GNN model for this sphereMNIST, with each stage conduct convolution on spherical mesh of a specific level. Detailed architecture (suppose mesh levels used are $Li,Lj,Lk$):

Conv(1,16)$_{Li}$-BN-ReLU-DownSamp-ResBlock(16,16,64)$_{Lj}$-DownSamp-ResBlock(64,64,256)$_{Lk}$-AvgPool-FC(10),

We use the 4-stage model architecture for SphereModelNet-40, where 4 mesh levels are: $L5,L4,L3,L2$. Detailed architecture are:

Conv(6,32)$_{L5}$-BN-ReLU-DownSamp-ResBlock(32,32,128)$_{L4}$-DownSamp\\-ResBlock(128,128,512)$_{L3}$-DownSamp-ResBlock(512,512,2048)$_{L4}$-DownSamp-AvgPool-FC(40),

where the GraphConv(feat\_in, feat\_out) in above model architectures can be either Mesh Convolution layer or Graph Convolution layer, and ``ResBlock'' is a bottleneck module with two $1\times1$ convolution layers and one GraphConv layer.

{\bf Training Details}
For SphereMNIST experiments, we use batch size of 64, Adam optimizer, initial learning rate of 0.01 which decays by 0.5 every 10 epoches. We totally train model for 100 epoches.

For SphereModelNet-40 experiment, we batch size of 16, Adam optimizer, initial learning rate of 0.005 which decay by 0.7 every 25 epoches. We totally train 300 epoches.

{\bf  Results on fine mesh}

Tab. \ref{tab:spheremnist-fine} show the results of SphereMNIST and Sphere-ModelNet40 on fine meshes on the sphere. 
Specifically, 
the mesh used for SphereMNIST here is of levels $L4,L3,L2$, 
and the SphereModelNet-40 mesh of levels $L5,L4,L3,L2$,
same as in \cite{ugscnn}.

\begin{table}[h]
\scriptsize
\centering
\caption{Results on SphereMNIST and SphereModelNet-40 following setup in \cite{ugscnn}}
\begin{tabular}{c|c|c} 
\hline 
Model & \tabincell{c}{SphereMNIST\\Acc}  & \tabincell{c}{SphereModelNet-40\\Acc} \\
\hline  
S2CNN \cite{s2cnn} & 96.0 & 85.0 \\
UGSCNN \cite{ugscnn}            & 99.2 & \bf{90.50} \\
GCN                  & 95.8 & 87.07\\
\hline
ChebNet($L$=4) & \bf{99.3} & 88.05 \\
ChebNet($L$=5)                          & - & 88.90 \\
ChebNet($L$=6)                          & - & 88.70 \\
ChebNet($L$=7)                          & - & 88.78 \\
\hline
L3Net (${1123}$)               & 99 .10 & 90.24 \\
L3Net (${112}$)              & 98.90 & 89.67 \\
\hline
\end{tabular}
\label{tab:spheremnist-fine}
\end{table}

\subsection{Facial Expression Recognition}\label{app:detail-2}
{\bf Landmarks setting}
15 landmarks are selected from the standard 68 facial landmarks defined in AAM \cite{aam}, 
and edges are connected according to prior information of human face, e.g., nearby landmarks on the eye are connected,
see Fig. \ref{fig:diag1} (left).

{\bf Dataset setup}

$\bullet$ CK+:

The CK+ dataset  \cite{ck+_dataset}  is the mostly used laboratory-controlled FER dataset (downloaded from:\\ \href{http://www.jeffcohn.net/resources/}{\textit{http://www.jeffcohn.net/resources/}}). It contains 327 video sequences from 118 subjects with seven basic expression labels(anger, contempt, disgust, fear, happiness, sadness, and surprise). Every sequence shows a shift from neutral face to the peak expression. We extract the last three frames from each sequence in the CK+ dataset, form a dataset with 981 samples. Every facial image is aligned and resized to $(120, 120)$ with face alignment model \cite{facealignment}, and then we use this model again to get facial landmarks. As we describe in Sec. 4.2, we select 15 from 68 facial landmarks and build graph on them. The input feature for each node is an image patch centered at the landmark with size $(20,20)$, concatenated with the landmark's coordinates, so the total input feature dimension is 402. 

$\bullet$ FER13:

FER13 dataset  \cite{fer13_dataset} is a large-scaled, unconstrained database collected automatically by Goole Image API (downloaded from: \href{https://www.kaggle.com/c/challenges-in-representation-learning-facial-expression-recognition-challenge/data}{\textit{https://www.kaggle.com/c/challenges-in-representation-learning-facial-expression-recognition-challenge/data}}). It contains 28,709 training images, 3589 validation images and 3589 test images of size $(48, 48)$ with seven common expression labels as CK+. We align facial images, get facial landmarks, and select nodes \& build graph the same way as we do in CK+. Input features are local image patch centered at each landmark with size $(8, 8)$ and landmark's coordinates, so the total input feature dimension is 66.

{\bf  Network architectures.}

$\bullet$ CK+:

GraphConv(402,64)-BN-ReLU-GraphConv(64,128)-BN-ReLU-FC(7),

$\bullet$ FER13:

GraphConv(66,64)-BN-ReLU-GraphConv(64,128)-BN-ReLU-GraphConv(128,256)-BN-ReLU-FC(7),

where GraphConv(feat\_in, feat\_out) here can be any type of graph convolution layer, including our L3Net.

{\bf Training details.} 

$\bullet$ CK+:

We use 10-fold cross validation as \cite{fer_cnn_ding2017}. Batch size is set as 16, learning rate is 0.001 which decay by 0.1 if validation loss remains same for last 15 epoches. We choose Adam optimizer and train 100 epoches for each fold validation.

$\bullet$ FER13:

We report results on test set. Batch size is set as 32, learning rate is 0.0001 which decay 0.1 if validation loss remains same for last 20 epoches. We choose Adam optimizer and train models for 150 epoches. 

{\bf Runtime analysis details.} 
In section 4.2, we report the running time of our L3Net(order 1,1,2,3), 13.02ms, and best ChebNet, 12.56ms, on CK+ dataset, which are comparable. Here, we provide more details about this. The time we use to compare is the time of model finishing inference on validation set with batch size of 16. For each model, we record all validation time usages in all folds and report the average of them. The Runtime analysis is performed on a single NVIDIA TITAN V GPU.

\subsection{Skeleton-based Action Recognition}\label{app:detail-3}

\begin{figure}[t]
    \centering
    \includegraphics[width=0.4\linewidth]{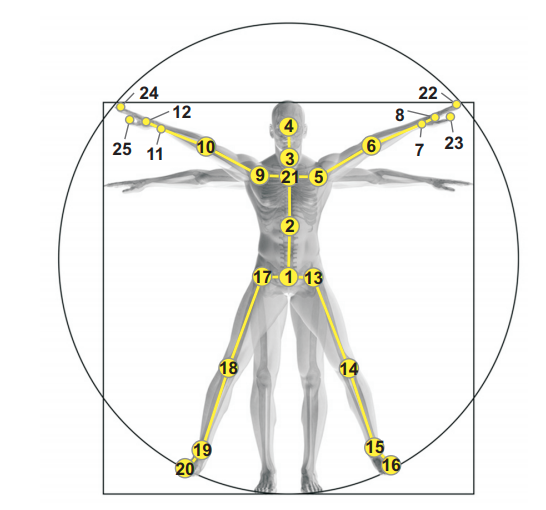}
    \caption{Illustration of 25-point body joints and graph.}
    \label{fig:body_Vitruvian_Man}
\end{figure}

{\bf Dataset setup.}

$\bullet$ NTU-RGB+D:

NTU-RGB+D \cite{shahroudy2016ntu} is a large skeleton-based action recognition dataset with three-dimensional coordinates given to every body joint (downloaded from: \href{http://rose1.ntu.edu.sg/datasets/requesterAdd.asp?DS=3}{\textit{http://rose1.ntu.edu.sg/datasets/requesterAdd.asp?DS=3}}). It comprises 60 action classes and total 56,000 action clips. Every clip is captured by three fixed Kineticsv2 sensors in lab environment performed by one of 40 different subjects. Three sensors are set at same height but in different horizontal views, $-45^{\circ}, 0^{\circ}, 45^{\circ}$. There are 25 joints tracked, as shown in Fig. \ref{fig:body_Vitruvian_Man}. Two experiment setting are proposed by \cite{shahroudy2016ntu}, cross-view (X-view) and cross-subject (X-sub). X-view consists of 37,920 clips for training and 18960 for testing, where training clips are from sensor on $0^{\circ}, 45^{\circ}$, testing clips from sensor on $-45^{\circ}$. X-sub has 40,320 clips for training and 16,560 clips for testing, where training clips are from 20 subjects, testing clips are from the other 20 subjects. We test our model on both settings.

$\bullet$ Kinetics:

Kinetics \cite{kay2017kinetics} is a large and most commonly-used action recognition dataset with nearly 300,000 clips for 400 classes (downloaded from: \href{https://deepmind.com/research/open-source/kinetics}{\textit{https://deepmind.com/research/open-source/kinetics}}). We follow \cite{stgcn} to get 18-point body joints from each frame using OpenPose \cite{cao2017realtime} toolkit. Input features for each joint to the Network is $(x, y, p)$, in which $x,y$ are 2D coordinates of the joint, and $p$ is the confidence for localizing the joint. To eliminate the effect of skeleton-based model's inability to recognize objects in clips, we mainly focus on action classes that requires only body movements. Thus, we conduct our experiments on Kinetics-Motion, proposed by \cite{stgcn}. This is a small dataset that contains 30 action classes strongly related to body motion. Note that there are severe data missing problem in landmark coordinates in Kinetics data, so we also use our regularization scheme in this experiment.

{\bf Network Architectures.}

$\bullet$ NTU-RGB+D:

We follow the architecture in \cite{stgcn}:

STGraphConv(3,64,9,s1)-STGraphConv(64,64,9,s1)-STGraphConv(64,64,9,s1)-STGraphConv(64,64,9,s1)-STGraphConv(64,128,9,s2)-STGraphConv(128,128,9,s1)-STGraphConv(128,128,9,s1)-STGraphConv(128,256,9,s2)-STGraphConv(256,256,9,s1)-STGraphConv(256,256,9,s1)-STAvgPool-fc(60).

$\bullet$ Kinetics:

We also design a computation-efficient architecture for Kinetics-Motion with larger temporal downsampling rate, which results in less forward time:

STGraphConv(3,32,9,s2)-STGraphConv(32,64,9,s2)-STGraphConv(64,64,9,s1)-STGraphConv(64,64,9,s1)-STGraphConv(64,128,9,s2)-STGraphConv(128,128,5,s1)-STGraphConv(128,128,5,s1)-STGraphConv(128,256,5,s2)-STGraphConv(256,256,3,s1)-STGraphConv(256,256,3,s1)-STAvgPool-fc(60),

where the structure of STGraphConv(feat\_in, feat\_out, temporal\_kernel\_size, temporal\_stride) is:

GraphConv(feat\_in, feat\_out)-BN-ReLU-1DTemporalConv(feat\_out, feat\_out, temporal\_kernel\_size, temporal\_stride)-BN-ReLU.

{\bf Training Details}

$\bullet$ NTU-RGB+D:

We use batch size of 32, initial learning rate of 0.001 which decay by 0.1 at (30, 80) epoch, and total train 120 epoches. SGD optimizer is selected. We padding every sample temporally with 0 to 300 frames.

$\bullet$ Kinetics:

We use batch size of 32, initial learning rate of 0.01 which decay by 0.1 at (40, 80) epoch, and total train 100 epoches. SGD optimizer is selected. We padding every sample temporally with 0 to 300 frames, and during training, we perform data augmentation by randomly choosing 150 contiguous frames.

\subsection{Details of experiment on MNIST}\label{app:detail-4}

\subsubsection{Simulated graph noise on $7\times7$ MNIST.}

Here we describe three types of noise in our experiments:

{\bf Gaussian noise}.
Given a $7\times7$ image from MNIST, we sample 49 values from $\mathcal{N}(0, std)$. the $std$ controls the strength of noise added. We conduct experiments under $std={0.1,0.2,0.3}$ as shown in Tab. \ref{tab:mnist7_gaussian}.
The amount of noise is also measured by PNSR which is standard for image data.

{\bf Missing value noise.}
Given a image, we randomly sample 49 values from $U(0,1)$, and select nodes with probabilities less than a threshold. This threshold is called $noise\_level$, which controls the percentage of nodes affected. Then, we remove the pixel value at those selected nodes. Experiments with $noise\_level=0.1,0.2,0.3$ are conducted.

{\bf Graph node permutation noise}.
For each sample, we randomly select a permutation center node which has exact 4 neighbors. Then, we rotate its neighbors clockwise by 90 degree, e.g., top neighbor becomes right neighbor, and then we update the indices of permuted nodes.

\begin{table}[t]
\centering
\begin{minipage}[t]{0.45\textwidth}
\tiny
\centering
\caption{Results on MNSIT with grid size $28\times28$, }
\begin{adjustbox}{max width=\textwidth}
\begin{tabular}{c|c|c|c} 
\hline 
Model & bases order & \tabincell{c}{\#params\\(w/o FC)} & Acc\\
\hline  
\hline
GCN       & 1         & 2.4k & $93.30\pm0.12$ \\
\hline
\multirow{7}*{ChebNet} & $L$=3     & 6.5k & $93.93\pm0.18$\\
~                      & $L$=4   & 8.6k & $94.97\pm0.06$\\
~                      & $L$=5   & 10.7k & $95.87\pm0.09$\\
~                      & $L$=6   & 12.8k & $96.64\pm0.12$ \\
~                      & $L$=7   & 14.8k & $96.98\pm0.19$ \\
~                      & $L$=9   & 19.0k & $97.43\pm0.14$ \\
~                      & $L$=15  & 31.5k & $\mathbf{97.91\pm0.08}$ \\
~                      & $L$=20  & 41.9k & $97.90\pm0.04$ \\
\hline
\multirow{2}*{L3Net}  & 1;1;2     & 41.0k & $96.78\pm0.08$\\
~                      & 1;1;2;3   & 79.2k & $97.32\pm0.10$\\
\hline
\end{tabular}
\label{tab:mnist-28}
\end{adjustbox}
\end{minipage}
\vspace{0.03\textwidth}
\begin{minipage}[t]{0.45\textwidth}
\tiny
\caption{Results on MNSIT with grid size $14\times14$}
\begin{adjustbox}{max width=\textwidth}
\begin{tabular}{c|c|c|c} 
\hline 
Model & \tabincell{c}{bases\\order} & \tabincell{c}{\#params\\(w/o FC)} & Acc\\
\hline
\hline
GCN          & 1         & 2.4k  & $93.70\pm0.09$\\
\hline
\multirow{5}*{ChebNet} & $L$=3     & 6.5k  & $96.06\pm0.16$\\
~ & $L$=4   & 8.6k  & $96.85\pm0.11$\\
~ & $L$=5 & 10.7k & $97.24\pm0.28$\\
~ & $L$=6 & 12.8k & $97.58\pm0.10$ \\
~ & $L$=7 & 14.9k & $\mathbf{97.74\pm0.07}$ \\
\hline
\multirow{4}*{\bf{L3Net}} & 0;1;2   & 13.3k & $97.17\pm0.09$ \\    
~          & 1;1;2   & 14.8k & $97.24\pm0.12$\\
~          & 1;1;2reg0.001   & 14.8k & $97.43\pm0.07$\\
~          & 1;1;2;3   & 25.1k & $97.51\pm0.07$\\
\hline
\end{tabular}
\label{tab:mnist-14}
\end{adjustbox}
\end{minipage}
\end{table}

\begin{table}[H]
\scriptsize
\centering
\caption{Results on MNSIT with grid size $7\times7$ with different levels of missing value}
\begin{tabular}{c|c|c|c||c||c|c|c} 
\hline 
Model & \tabincell{c}{bases\\order} & reg & \tabincell{c}{\#params\\(w/o FC)} & Acc(original) & Acc(psnr 18.70) & Acc(psnr 15.33) & Acc(psnr 13.15)\\
\hline
\hline
GCN          & 1         & - & 2.4k  & $90.02\pm0.24$ & $83.44\pm0.15$ & $77.23\pm0.13$ & $71.67\pm0.06$ \\
\hline
\multirow{5}*{ChebNet} & $L$=3     & - & 6.5k  & $92.85\pm0.09$ & $87.09\pm0.18$ & $82.11\pm0.18$ & $76.15\pm0.26$ \\
~                       & $L$=4   & - & 8.6k  & $93.12\pm0.1$ & $87.09\pm0.16$ & $82.22\pm0.28$ & $75.95\pm0.22$\\
~                       & $L$=5 & - & 10.7k & $93.2\pm0.07$ & $87.01\pm0.14$ & $82.04\pm0.14$ & $76.21\pm0.38$\\
~                       & $L$=6 & - & 12.7k & $93.42\pm0.09$ & $87.20\pm0.3$ & $81.19\pm0.29$ & $75.24\pm0.32$\\
~                       & $L$=7 & - & 14.8k & $93.45\pm0.06$ & $87.08\pm0.11$ & $81.00\pm0.17$ & $75.31\pm0.34$\\
\hline
\multirow{4}*{\bf{L3Net}} & 1;1;2 & - & 8.4k  & $93.56\pm0.08$ & $86.64\pm0.16$ & $81.14\pm0.30$ & $75.07\pm0.08$ \\
~                       & 1;1;2     & 0.5 & 8.4k  & $93.85\pm0.13$ & $87.22\pm0.23$ & $\mathbf{82.84\pm0.11}$ & $\mathbf{76.48\pm0.23}$ \\
~                       & 1;1;2;3   & - & 12.2k  & $93.67\pm0.15$ & $86.51\pm0.38$ & $80.68\pm0.11$ & $74.24\pm0.36$ \\
~                       & 1;1;2;3   & 0.5 & 12.2k  & $\mathbf{93.85\pm0.15}$ & $\mathbf{87.22\pm0.08}$ & $82.64\pm0.31$ & $76.08\pm0.38$ \\
\hline
\end{tabular}
\label{tab:mnist7-speckle}
\end{table}

\subsubsection{Network architecture and training details}
We use the same architecture for different experiment settings:

GraphConv(1,32)-BN-ReLU-GraphConv(32,64)-BN-ReLU-FC(10),

where GraphConv can be different types of graph convolution layers.We set batch size to 100, use Adam optimizer, and set initial learning rate to 1e-3. Learning rate will drop by 10 if the least validation loss remains the same for the last 15 epoches. We set total training epoches as 200.

\subsubsection{Additional results}
Here, we show experiments results on $28\times28, 14\times14$ grid, as well as $7\times7$ grid with missing values. Tab. \ref{tab:mnist-28} shows results on $28\times28$ image grid. Our model have better performance than other methods. 

Tab. \ref{tab:mnist-14} shows results on $14\times14$ image grid, where our L3Net have comparable results with the best ChebNet \cite{chebnet} method.

We shows our results on $7\times7$ image grid with missing values in Tab. \ref{tab:mnist7-speckle}. With regularization, L3Net achieves the best performance in every experiment with different noise levels.

\end{document}